\documentclass{article}

\newcommand{\haim}[1]{}
\newcommand{\jay}[1]{}
\newcommand{\yishay}[1]{}
\newcommand{\uri}[1]{}

\usepackage[utf8]{inputenc} 
\usepackage[T1]{fontenc}    
\usepackage{hyperref}       
\usepackage{url}            
\usepackage{booktabs}       
\usepackage{amsfonts}       
\usepackage{nicefrac}       
\usepackage{microtype}      
\usepackage{xcolor}         
\usepackage{enumerate}
\usepackage[english]{babel}
\usepackage[nottoc]{tocbibind}

\usepackage{amsthm}
\usepackage{xspace}
\usepackage{commath}
\usepackage{float}
\usepackage[linesnumbered, boxed, ruled]{algorithm2e}
\usepackage{csquotes}
\usepackage{mathtools}
\usepackage{tablefootnote}

\newcommand{\eps}{\varepsilon}
\let\oldnl\nl
\newcommand{\nonl}{\renewcommand{\nl}{\let\nl\oldnl}}

\newcommand{\ip}[2]{\left\langle{#1},{#2}\right\rangle}

\newcommand{\naturals}{\mathbb{N}}
\newcommand{\reals}{\mathbb{R}}

\DeclareMathOperator*{\argmax}{arg\,max}

\newcommand{\CSTA}{\mbox{CSTA}\xspace}

\newcommand{\psco}{\mbox{PSCO}}
\usepackage{tikz}
\usetikzlibrary{arrows}
\tikzset{
  treenode/.style = {align=center, inner sep=0pt, text centered,
    font=\sffamily},
  arn_n/.style = {treenode, circle, white, font={\fontsize{0.1pt}{0.1}\selectfont}, draw=black,
    fill=black, text width=3em, minimum width=3em},
  arn_r/.style = {treenode, circle, black, font={\fontsize{0.1pt}{0.1}\selectfont}, draw=red,
    fill=red, text width=3em, minimum width=3em},
  arn_s/.style = {treenode, circle, black, font={\fontsize{0.1pt}{0.1}\selectfont}, draw=black,
    fill=white, text width=3em, minimum width=3em},
  small_arn_n/.style = {treenode, circle, white, font={\fontsize{0.1pt}{0.1}\selectfont}, draw=black,
    fill=black, text width=1.5em, minimum width=1.5em},
  small_arn_r/.style = {treenode, circle, black, font={\fontsize{0.1pt}{0.1}\selectfont}, draw=red,
    fill=red, text width=1.5em, minimum width=1.5em},
  small_arn_s/.style = {treenode, circle, black, font={\fontsize{0.1pt}{0.1}\selectfont}, draw=black,
    fill=white, text width=1.5em, minimum width=1.5em},
}

\usepackage{microtype}
\usepackage{graphicx}
\usepackage{subfigure}
\usepackage{booktabs} 

\usepackage{hyperref}

\usepackage[accepted]{icml2023}

\usepackage{amsmath}
\usepackage{amssymb}
\usepackage{mathtools}
\usepackage{amsthm}

\usepackage[capitalize,noabbrev]{cleveref}

\theoremstyle{plain}
\newtheorem{theorem}{Theorem}[section]

\newtheorem{lemma}[theorem]{Lemma}
\newtheorem{corollary}[theorem]{Corollary}
\theoremstyle{definition}
\newtheorem{definition}[theorem]{Definition}
\newtheorem{assumption}[theorem]{Assumption}
\theoremstyle{remark}

\usepackage[textsize=tiny]{todonotes}

\icmltitlerunning{Concurrent Shuffle Differential Privacy Under Continual Observation}

\begin{document}

\twocolumn[
\icmltitle{Concurrent Shuffle Differential Privacy Under Continual Observation}



\icmlsetsymbol{equal}{*}


\begin{icmlauthorlist}
\icmlauthor{Jay Tenenbaum}{jayaff}
\icmlauthor{Haim Kaplan}{haimaff}
\icmlauthor{Yishay Mansour}{yishayaff}
\icmlauthor{Uri Stemmer}{uriaff}
\end{icmlauthorlist}

\icmlaffiliation{jayaff}{Google Research. \href{mailto:jayten@google.com}{jayten@google.com}.}
\icmlaffiliation{haimaff}{Blavatnik School of Computer Science, Tel Aviv University and Google Research. \href{mailto:haimk@tau.ac.il}{haimk@tau.ac.il}.}
\icmlaffiliation{yishayaff}{Blavatnik School of Computer Science, Tel Aviv University and Google Research. \href{mailto:mansour.yishay@gmail.com}{mansour.yishay@gmail.com}.}
\icmlaffiliation{uriaff}{Blavatnik School of Computer Science, Tel Aviv University and Google Research. \href{mailto:u@uri.co.il}{u@uri.co.il}}


\icmlkeywords{differential privacy, shuffle, continual observation, contextual linear bandits, concurrent}

\vskip 0.3in
]



\printAffiliationsAndNotice{}  

\begin{abstract}

We introduce  the \textit{concurrent shuffle} model of differential privacy.
In this model we have multiple concurrent shufflers permuting messages from  different, possibly overlapping, batches of users. 
Similarly to the standard (single) shuffle model, the privacy requirement is that the concatenation of all shuffled messages should be differentially private.

We study the {\em private continual summation} problem (a.k.a.\ the {\em  counter} problem) and show that 
the  concurrent shuffle model allows for significantly improved error compared to a standard (single) shuffle model. Specifically,
%
we give a summation algorithm with error $\Tilde{O}(n^{1/(2k+1)})$ 
with $k$ concurrent shufflers on a sequence of length $n$. Furthermore, we prove that this bound is tight for any $k$, even if the algorithm can choose the sizes of the batches adaptively. For  $k=\log n$ shufflers, the resulting error is polylogarithmic, much better than $\Tilde{\Theta}(n^{1/3})$ which we show is the smallest  possible with a single shuffler.
%
%

We use our online summation algorithm to get algorithms with improved regret bounds for the contextual linear bandit problem.
In particular we get optimal  $\Tilde{O}(\sqrt{n})$  regret with  $k= \Tilde{\Omega}(\log n)$ concurrent shufflers.
%
%
\end{abstract}

\section{Introduction}
Differential privacy, introduced in the seminal work of \citet{dwork2006calibrating}, is a formal notion of privacy that enables the release of statistical information about a set of users, without compromising the privacy of any individual user. Briefly, differential privacy requires that any change in a single user's data changes the probability of any  output of the computation by only a limited amount. 
\textit{Differential privacy} has been extensively studied under many different sub-models of privacy. On one end of the spectrum lies the \textit{centralized model} where the users trust the server who holds a database of their data and is liable  to protect it. This server
 receives a (maybe interactive) sequence of queries to the database, and must make sure that the  published aggregated statistics respect the privacy constraints.
 On the other end of the spectrum lies the (strictly stronger) \textit{local model} of differentialy privacy (LDP), where each user privatizes its own data prior to sending it to the server.  In the middle between the local and the centralized models lies the \textit{shuffle model} \citep{bittau2017prochlo,cheu2019distributed,ghazi2019scalable,balle2019privacy,erlingsson2019amplification, erlingsson2020encode}. This model introduces a trusted shuffler (or \textit{shuffler} for short) that receives (encoded) messages from the users and permutes them (i.e., disassociates a message from its sender) before they are delivered to the server.\footnote{
For privacy analysis, we assume that the shuffle is perfectly secure, i.e., its output contains no information about which user generated each of the messages. This is traditionally achieved by the shuffler stripping implicit metadata from the messages (e.g., timestamps, routing information), and frequently forwarding this data to remove time and order information.} Intuitively, protocols in the shuffle model ensure that sufficiently many reports are collected and shuffled together so that any one report can hide in a shuffled batch.

Motivated by real-world applications, such as monitoring systems or search trends which output continual statistics,
\citet{dwork2010differential} studied ``differential privacy under continual observation''. They focus on the fundamental problem of private summation under continual observation (\psco) in which the server  continuously reports an approximation of an accumulating sum  of bounded values in $[0,1]$, one from each new user.
They guarantee event-level privacy, which protects the content of a single user in a single interaction with the server.
Over $n$ users, \citet{dwork2010differential} and \citet{chan2011private} reduced the  trivial $O(\sqrt{n}/\eps)$ error achieved by advanced composition (and is optimal in LDP \citet{vadhan2017complexity}), to $O(polylog(n)/\eps)$  using a tree-based algorithm.  Briefly, they maintain  a binary tree with leaves $1,\ldots,n$, where the $t$'th leaf contains the value of the $t$'th user, and each \textit{internal node} (not a leaf or the root) contains the sum of the leaves in its subtree with added noise $Lap(\Tilde{\Theta}(1/\eps))$ to ensure privacy. The error bound follows since any partial sum is computed with $\leq \log n$ internal nodes.
\footnote{Advanced composition of differential privacy: the combination of $k$ $(\eps, \delta)$-differentially private
algorithms is $O\left(k\eps^2+2\sqrt{k\log(1/\delta')}, k\delta+\delta'\right)$-differentially private for any $\delta'>0$ \citep{dwork2010boosting}.} Note that private (non-continual) summation has optimal $\tilde{\Theta}(1/\eps)$ \citep{dwork2006calibrating} and $\tilde{\Theta}(\sqrt{n}/\eps)$ \citep{beimel2008distributed,chan2012optimal} errors in the centralized and local models, respectively.\footnote{The $\tilde{\Theta}(\cdot)$ hides polylogarithmic terms in $n,\frac{1}{\delta},\frac{1}{\beta}$, etc.} In this paper, we apply the shuffle model in the continual observation setting, to get smaller errors. 


The first technicality that must be addressed is that, traditionally, shuffle model protocols assume only a single interaction with the shuffler. That is, most existing protocols assume that {\em all} of the users submit their (encoded) messages to the shuffler {\em simultaneously}. This is, of course, unsuitable for {\em online} or {\em interactive} settings in which the users arrive one by one in a sequential manner. 
To address this and to adapt the shuffle model to online learning settings, \citet{tenenbaum2021differentially}, \citet{cheu2021shuffle} and considered a model in which the arriving users are partitioned into consecutive batches, and each batch of users interacts with the shuffler separately. Specifically, in their model, when a user arrives it submits its (encoded) message to the shuffler, that stores all the (encoded) messages it receives until the current batch ends. When the batch ends, the shuffler reveals to the server a random permutation of all the (encoded) massages it received from the users of this batch, and the server post-processes this in order to estimate the desired statistics. We call this the \textit{sequential shuffle} model.
In this model, using the techniques above and selecting the optimal batch size for the \psco{} problem, gives an algorithm with an error of $\Tilde{O}(n^{1/3})$.


\subsection{Our contributions}
$\bullet$  {\bf The concurrent shuffle model:}  We ask \enquote{can we leverage a small number of shufflers working concurrently to design an improved algorithm for the \psco{} problem (with error less than $\Tilde{O}(n^{1/3})$)?}
    We introduce a novel model of differential privacy under continual observation, which we call the \textit{Concurrent Shuffle} model.
    In this model, we have  $k\in \naturals$ different shufflers working concurrently. At each time, a user arrives and sends information to a subset of the $k$ shufflers. Once a shuffler is \textit{full}, i.e.,\  all its users have arrived, it sends its  data to the server in a random order.
    Then this shuffler is reset and  reused for the subsequent users.

$\bullet$   {\bf Algorithm for \psco:}
A weakness of the sequential shuffle model of 
\citet{tenenbaum2021differentially,cheu2021shuffle} is that 
the server does not get 
any information about
the users in the current batch until it fills up. Thus, intuitively, the error increases with the size of the batch. On the other hand when we try to address this by reducing the size of the batch, then the number of batches increases. This intuitively also increases the error, because for each batch we must introduce sufficient noise in order to hide any single user from this batch, and hence the overall noise increases with the number of batches.
The $\tilde{O}(n^{1/3})$ bound for the \psco{} problem comes from balancing these two sources of error. 

We show how to improve this tradeoff using concurrent  shufflers. Even with $k=2$ concurrent shufflers we reduce the error significantly. The idea is to use two concurrent shufflers of different sizes. The larger one allows us to cover a substantial part of the already arrived prefix using a small number of batches, while the smaller one allows us reduce the number of users we are blind to because we still did not get the reports from the current shufflers. 
As $k$ increases we can better control this tradeoff. Consequently,
for \psco{} with $k$ concurrent shufflers we derive an {\em optimal} algorithm  with worst-case error $\Tilde{O}\left(k^{3/2}n^{\frac{1}{2k+1}}\right)$ with high probability. In particular, with $k=\log(n)$ concurrent  shufflers we recover the same $polylog(n)$ error achievable in the centralized model of DP. This is a significant improvement over the $\tilde{O}(n^{1/3})$ error achievable in the sequential shuffle model (which we show is optimal for the sequential shuffle model). 
One view can our algorithm as variation of the tree-based algorithm of \citet{dwork2010differential}, with different internal node degrees and a different method of approximating the running sum.

    
$\bullet$  {\bf Lower bound for binary summation:}   For  private summation of bounded values,  and specifically binary values, we show a lower bound which matches the performance of our algorithm {\em for any value of $k$}. That is, for any number $k$ of concurrent shufflers, we prove that no concurrent shuffle mechanism can have smaller error than what we achieve! This lower bound holds even for  algorithms that can choose the mechanisms (encoders and sizes) adaptively.
In particular our proof implies that the algorithm with $\tilde{O}(n^{1/3})$ error is optimal in the sequential shuffle model (in which $k=1$).


$\bullet$ {\bf  Application to linear contextual bandits:} We use our private summation algorithm to devise a new algorithm for contextual linear bandits with arbitrary contexts (even adversarial) in the concurrent shuffle model. Our algorithm is based on LinUCB \citep{abbasi2011improved} and achieves regret $\Tilde{O}\left(k^{3/4}n^\frac{k+1}{2k+1}/\sqrt{\eps}\right)$ with $k$ concurrent shufflers. Specifically, for  $k=\log n$, this regret is $\Tilde{O}(\sqrt{n}/\sqrt{\eps})$, which clearly has optimal asymptotic dependence on $n$ even without privacy. This improves over the result of \citet{chowdhury2022shuffle} who obtained regret of $\Tilde{O}(n^{3/5}/\sqrt{\eps})$ in the sequential shuffle model. That is, we reduce the regret to optimal, from $\Tilde{O}(n^{3/5}/\sqrt{\eps})$ to $\Tilde{O}(\sqrt{n}/\sqrt{\eps})$, while essentially maintaining the same trust model (small number of concurrent shufflers instead of just one). This answers positively the open question of \citet{chowdhury2022shuffle}.

We believe that the concurrent shuffle model may allow to get improved utility (compared to the sequential shuffle model) for additional problems.

All missing proofs and algorithms appear in the Appendix.

\subsection{Further related work}
A large body of research has recently studied the \psco{} problem with horizon $n$ in several different settings. 
In Local DP (LDP), \citet{vadhan2017complexity} give an algorithm and a matching lower bound for error $\Tilde{O}(\sqrt{n}/\eps)$.
In the centralized model, 
\citet{henzinger2022constant} study fine-grained error bounds, and \citet{ghazi2022differentially} give upper and lower bounds for more general tree structures.
Other works considered $\ell_p$ extensions, high dimensional variants, or improvements and applications of \psco{}.\footnote{\citet{bolot2013private,smith2017interaction,fichtenberger2021differentially,kairouz2021practical,upadhyay2021framework,upadhyay2021differentially,denissov2022improved,jain2021price,cardoso2022differentially,epasto2023differentially,henzinger2023almost}}

Several works have studied the private multi-armed bandit problem \citep{mishra2015nearly,tossou2017achieving,sajed2019optimal,ren2020multi,chen2020locally,zhou2021local,dubey2021no}, the private contextual linear bandit problem \citep{shariff2018differentially, zheng2020locally, han2020sequential, ren2020batched,garcelon2022privacy}, and the more general private reinforcement learning \citep{vietri2020private,garcelon2021local,chowdhury2022differentially} problem, in both local and centralized models of privacy. 
The regret gap between the two models (when the contexts are arbitrary, not stochastic \citep{han2021generalized}) has shrunk using the intermediate sequential shuffle model \citep{tenenbaum2021differentially,chowdhury2022shuffle,garcelon2022privacy}.
See Section~\ref{sec:contextualLinearBandits} for further discussion of these results for private contextual linear bandits.

If we are not in the continual observation privacy model then it is known that  several shuffle mechanisms 
can run in parallel using a single shuffler \citep{cheu2021shuffle,cheu2021differential}.
This is not true when the adversary continually observes reports of the algorithm. In fact we 
 show that using  $k$ concurrent shufflers substantially improves accuracy under continual observation.

\section{Background and preliminaries}\label{sec:prelim}
In this paper, we study the power of using concurrent shufflers for problems with continual reporting, focusing specifically on private summation under continual observation (\psco).
We first 
define the problem of \psco, then present the tree-based mechanism that solves it with polylogarithmic error, and finally review
 the shuffle model of differential privacy.

\subsection{Private summation under continual observation}
\looseness=-1
Motivated by monitoring event occurrences over time, \citet{dwork2010differential} studied the problem of 
private continual reporting the running sum of 
bounded real values that arrive online.
Specifically, they considered binary values and reported the number of 1's that arrived by time $t$ while protecting the value reported at each particular time.
Recall the definition of \citet{dwork2010differential} which bounds the additive errors of the released sum estimates throughout all the times:
\begin{definition}
A randomized streaming algorithm yields an $(n, \alpha, \beta)$ \textbf{summation algorithm} if for every input $b=(b_1,\ldots,b_n)$, it holds that\\ 
$\Pr\left[\forall t=1,\ldots,n~\abs{\hat{S}_{t}-\Sigma_{i=1}^{t}b_i}\leq \alpha\right]\geq 1-\beta,
$
where $\hat{S}_{t}$ denotes the random estimate released by the algorithm after observing $b_1,\ldots,b_t$.
\end{definition}

\subsection{The tree-based mechanism}

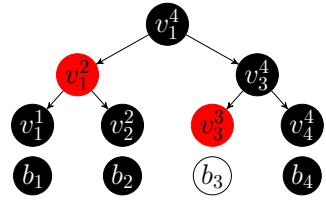
\begin{figure}
\centering
\resizebox{0.25\textwidth}{!}{
\begin{tikzpicture}[->,>=stealth',level/.style={sibling distance = 3.2cm/#1,
  level distance = 0.9cm}] 
\node [small_arn_n] {\Large $v_{1}^{4}$}
    child{ node [small_arn_r] {\Large $v_{1}^{2}$} 
            child{ node [small_arn_n] {\Large $v_{1}^{1}$} 
            	child{ node [small_arn_n] {{\Large $b_1$}} edge from parent[draw=none]}
            }
            child{ node [small_arn_n] {\Large $v_{2}^{2}$}
							child{ node [small_arn_n] {\Large $b_2$} edge from parent[draw=none]}
            }                            
    }
    child{ node [small_arn_n] {\Large $v_{3}^{4}$}
            child{ node [small_arn_r] {\Large $v_{3}^{3}$} 
							child{ node [small_arn_s] {\Large $b_3$} edge from parent[draw=none]}
            }
            child{ node [small_arn_n] {\Large $v_{4}^{4}$}
                                child{ node [small_arn_n] {\Large $b_4$} edge from parent[draw=none]}
            }
    }
; 
\end{tikzpicture}
}
\caption{The tree-based mechanism. Each internal node $v_{i}^{j}$ contains an estimate $S_{v_{i}^{j}}$ of the sum $\sum_{l=i}^{j}b_l$.
At the 3'rd time (when $b_3$ participates), we approximate $\sum_{i=1}^{3}b_i$ by $S_{v_{1}^{2}}+S_{v_{3}^{3}}$. The approximation errors in each $S_{v_{i}^{j}}$ are iid $\eta\sim L(\frac{\log n}{\eps})$.}
\label{fig:treeBasedMechanism}
\end{figure}

For the task of private summation in the centralized model, \citet{dwork2010differential} devised a tree-based mechanism. They assumed for simplicity that $n = 2^i$, and let $T$ be a complete binary tree with its leaf nodes associated with the values $b_1,\ldots, b_n$. Each node $x \in T$ stores the sum of all the leaf nodes in the subtree rooted at $x$. They observed that one can compute any partial sum $\sum_{j=1}^{t}b_j$ using the sums in at most $p := \log(n)$ nodes of $T$, and that for any two neighboring data sequences $D$ and $D_0$, the  sums stored at no more than $p$ nodes in $T$ are different.
They also observed that the error at each time is a sum of at most $p=\log (n)$ noises.
Hence, to make sure that the entire
tree is $(\eps,\delta)$-DP, by simple composition it suffices to ensure that each node preserves $(\frac{\eps}{p},\frac{\delta}{p})$-DP. By applying a concentration inequality over the $p$ noises and a union bound over the horizon $n$, they conclude that this algorithm creates a summation algorithm with an error term $\alpha$ which grows like $\Tilde{O}\left(1/\eps\right)$. An illustration of the tree-based mechanism appears in Figure~\ref{fig:treeBasedMechanism}.

\subsection{Shuffle-model privacy}

In the well-studied shuffle model of privacy, there are $n$ users, each with data $x_i \in X$. Each user applies some encoder $E : X \to Y^*$ to their data and sends the messages $E(x_i)=(y_{i,1},...,y_{i,p}) $ to a shuffler $S$. The shuffler then shuffles all the messages $y_{i,j}$ from all the users, and outputs them in a uniformly random order to the server.
Thus, the shuffle mechanism is defined by the pair  $(E, n)$.\footnote{Traditionally, the  algorithm which analyzes the shuffled output was part of the mechanism, but in this paper the server analyzes the outputs of many shufflers together.} We say that such a mechanism $M$ is \textit{$(\eps, \delta)$-shuffle differentially private} (or \textit{$(\eps, \delta)$-SDP} for short) if the shuffler's output is $(\eps, \delta)$-differentially private, or more formally:
A mechanism $M=(E,n)$ is $(\eps,\delta)$-SDP if for any pair of inputs $\{x_i\}_{i=1}^{n}$ and $\{x_i'\}_{i=1}^{n}$ which differ in at most one value, we have for all $B\subseteq Y^*$:
\[\Pr[S((E(x_i))_{i=1}^{n})\in B]\leq e^\eps\cdot \Pr[S((E(x_i'))_{i=1}^{n})\in B]+\delta,\]
where $(z_i)_{i=1}^{n}$ denotes concatenation of values $z_1,\ldots,z_n$.

\section{Concurrent shuffle differential 
privacy 
}\label{sec:concurrentshuffler}
To adapt the shuffle model to adaptive algorithms (e.g., bandits, sum estimates etc.) under continual observation, \citet{tenenbaum2021differentially,cheu2021shuffle} and \citet{chowdhury2022shuffle} divide the users into continuous batches, and run a shuffle-DP (SDP) mechanism over each batch separately. 
When a new batch starts, the server  selects 
the next shuffle  mechanism (encoder and size), possibly as a function of the outputs of the previous shuffle mechanisms (i.e.,\ it may be adaptive).
We refer to this as the \textit{sequential shuffle} model.
 
We consider a generalized model of privacy under continual observation which has similar privacy guarantees, but  is more powerful and enables improved accuracy. In this model, 
a set of $k\in \naturals$ different shufflers are
executed concurrently throughout the run of the algorithm, and each user can participate in several different shuffle mechanisms. We can reuse each shuffler again and again, each time possibly with a different mechanism.

Each shuffler $i=1,\ldots,k$, 
at each time, can either be \textit{active} or \textit{inactive}.
The active shufflers are in the process of accumulating messages from  the users.
Once an active shuffler fills up (i.e.\ 
all its users have sent it their encoded data), it is \textit{executed} and its  data is sent to the server in a random order. Then it is marked \textit{inactive}, until the server decides to reuse it on a new batch. Note that in this model each user $t$ can participate in the (up to $k$) active shufflers at time $t$.
As a simplifying assumption, we assume that the server is deterministic, i.e., 
the decisions to allocate new mechanisms and their sizes, and the statistic estimates are all deterministic functions of the outputs of the previously executed shuffle mechanisms. (The shuffle mechanisms are randomized though).
We formalize this in Algorithm~\ref{alg:framework} below.

\begin{algorithm}[htb]
\SetAlgoLined
\textbf{Input}:  $k$  shufflers.

\textbf{Processing:} for each time $t=1,\ldots,n$:

\vspace*{-0.3cm}
\begin{enumerate}
\setlength{\itemsep}{1pt}
    \item The server activates zero or more in active shufflers. It assigns each of them a shuffle mechanism (encoder $E$ and size $m$) that may depend on the outputs of previously executed shuffle mechanisms.\;    
    \item For each active shuffle mechanism $M=(E,m)$, the $t$'th user sends the encoding $E(b_t)$ of its private value $b_t$ to the shuffler.\;
    \item Each shuffle mechanism that has become full, sends the server its shuffled received data, and is deactivated.\;
    \item The server outputs an estimate $\hat{y}_t$.\; 
\end{enumerate}
 \caption{Concurrent Shuffle Differential Privacy Under Continual Observation with \textit{k} Shufflers}
 \label{alg:framework}
\end{algorithm}

\vspace*{-0.1cm}
\noindent
{\bf Privacy.}
Recall that the private data of the $t$'th user is its value $b_t$. In the sequential shuffle model \citep{tenenbaum2021differentially,cheu2021shuffle,chowdhury2022shuffle}, since each user participates in exactly one shuffle mechanism, we ensure $(\eps,\delta)$ differential privacy 
of the entire algorithm by making
each executed shuffle mechanism  $(\eps,\delta)$-SDP.


In our model,  a user can participate in several different shuffle mechanisms. Thus ensuring that each  shuffle mechanism is $(\eps,\delta)$-SDP is not enough to guarantee that the overall algorithm is $(\eps,\delta)$-DP.
%
Specifically, for a given input $\{x_i\}_{i=1}^{n}$, consider the (adaptive) sequence of shuffle mechanisms $(E_j,m_j)$ and shufflers $S_j$, run on batches of users $B_j$ ($|B_j|=m_j$) by some algorithm $A$.
Let $A(\{x_i\}_{i=1}^{n})\in Y$ be the  vector $S_j((E_j(u))_{u \in B_j})$ generated by concatenating the outputs of all these shuffle mechanisms.\footnote{Note that output of the algorithm throughout all times is a post-processing of $A(\{x_i\}_{i=1}^{n})$.} We say that $A$ is \textit{$(\eps,\delta)$-concurrent shuffle differentially private (CSDP)} if  for any pair of inputs $\{x_i\}_{i=1}^{n}$ and $\{x_i'\}_{i=1}^{n}$ which differ in at most one value, we have for all $B\subseteq Y$:$$\Pr[A(\{x_i\}_{i=1}^{n})\in B]\leq e^\eps\cdot \Pr[A(\{x_i'\}_{i=1}^{n})\in B]+\delta.$$

The main trade-off we address in this paper is how the error (the smaller the better) depends on  the  number of shufflers $k$ that the algorithm uses.

\looseness=-1
We demonstrate the power of the concurrent shuffle model on the specific problem of private summation, where each user has a private bounded value, and we estimate the running sum.

\subsection{Private summation for concurrent shufflers}\label{sec:concurrentshufflerPrivateSum}
Consider the \psco{} problem, where the value of the $t$'th user is  $\in [0,1]$, and in every time $t$ we estimate the sum $\sum_{s=1}^{t}b_s$ while protecting the values of the users.

In the sequential shuffle model (one shuffler), using a standard summation mechanism in the shuffle model with $\Tilde{O}(n^{2/3})$ batches of  size $\Tilde{O}(n^{1/3})$ attains error $\Tilde{O}(n^{1/3})$ (see Section~\ref{sec:constrainedShufflers}). In Section~\ref{sec:lb}, we show  that this is optimal for the sequential shuffle model.


We now give an algorithm in the concurrent shuffle model with lower (polylogarithmic) error, assuming we have $k=\Theta(\log n)$ shufflers. In Section~\ref{sec:constrainedShufflers}, we extend this algorithm 
for any $k \ll \log n$ 
shufflers and analyze its error.

\subsubsection{Tree-based algorithm in the concurrent shuffle model}\label{sec:treeBasedConcurrentShuffle}
In this section, we adapt the well known tree-based algorithm \citep{dwork2010differential,chan2011private}, originally built for private summation in the centralized model of differential privacy, to private summation in the concurrent shuffle model, to give an algorithm with additive error term $\alpha=\Tilde{O}\left(1/\eps\right)$.

We have a balanced binary tree in which the $t$'th leaf corresponds to the data of the user arriving at time $t$.
With each internal node, 
we associate 
a batch
that contains all users in its subtree. 
These batches are 
processed by $k=\log n - 1$ shufflers,
one for all nodes in a single level of the tree.
We compute a noisy sum of the batch associated with node $v$ using a shuffle-DP (SDP) summation mechanism $M_v^{sum}$. 
We assume that by analyzing the output of $M_v^{sum}$, we can get an unbiased estimate of the sum with sub-gaussian additive error, 
 which does not depend on the input. 
\begin{definition}\label{def:vstartForBinaryTree}
    For $t=1,\ldots,n$, let $V^{\star}_t$ be the set of all the highest level internal nodes in the tree that all their inputs have arrived and the shuffler associated with them completed. That is, the set of all left siblings of internal nodes in the path from the root to the $(t+1)$'th leaf.
\end{definition} 
To estimate the sum at time $t$, we sum the estimates of the set of nodes $V^{\star}_t$.
Algorithm~\ref{alg:summation} gives a formal description, and  Figure~\ref{fig:treeBasedMechanismShufflers} illustrates it and the definition of $V^{\star}_t$.

\begin{figure}
\centering
\resizebox{0.4\textwidth}{!}{
\begin{tikzpicture}[->,>=stealth',level/.style={sibling distance = 4cm/#1,
  level distance = 0.9cm}] 
\node [small_arn_n] {}
    child{ node [small_arn_r] {{\large $v_{1}^{4}$}} 
            child{ node [small_arn_n] {{\large $v_{1}^{2}$}} 
            	child{ node [small_arn_n] {{\large $b_1$}}}
                    child{ node [small_arn_n] {{\large $b_2$}}}
            }
            child{ node [small_arn_n] {{\large $v_{3}^{4}$}}
							child{ node [small_arn_n] {\large $b_3$}}
                                child{ node [small_arn_n] {{\large $b_4$}}}
            }                            
    }
    child{ node [small_arn_n] {{\large $v_{5}^{8}$}}
            child{ node [small_arn_r] {{\large $v_{5}^{6}$}} 
							child{ node [small_arn_n] {\large $b_5$}}
                                child{ node [small_arn_n] {{\large $b_6$}}}
            }
            child{ node [small_arn_n] {{\large $v_{7}^{8}$}}
                                child{ node [small_arn_s] {\large $b_7$}}
                                child{ node [small_arn_n] {{\large $b_8$}}}
            }
    }
; 
\end{tikzpicture}
}
\caption{The tree-based mechanism for $k=2$ shufflers. Each internal node $v_{i}^{j}$ contains an estimate $S_{v_{i}^{j}}$ of the sum $\sum_{l=i}^{j}b_l$ using a shuffle mechanism for summation. Internal nodes of height 1 ($v_{1}^{2},v_{3}^{4},v_{5}^{6}$, and $v_{7}^{8}$) and of height 2 ($v_{1}^{4}$ and $v_{5}^{8}$) apply shufflers of size 2 and 4 respectively to the data at their leaves. In times 6 and 7, $V^{\star}_6,V^{\star}_7=\{v_1^4,v_5^6\}$, hence we give the sum approximate $S_{v_{1}^{4}} + S_{v_{5}^{6}}$, and in time 8, $V^{\star}_8=\{S_{v_{1}^{4}},S_{v_{5}^{8}}\}$.}
\label{fig:treeBasedMechanismShufflers}
\end{figure}
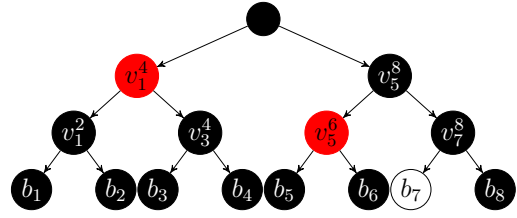

\begin{algorithm}[htb]
\SetAlgoLined

\textbf{Tree.} Instantiate a binary tree with leaves corresponding to  the numbers $1,\ldots,n$.

\textbf{Internal Nodes.} Each internal node (not the root or a leaf) $v$ in the binary tree, is associated with the set $L_v$ of its leaf  descendants.

\textbf{Processing.} For every time step $t\in \{1,\ldots,n\}$ corresponding to the $t$'th leaf in the tree:

\vspace{-0.2cm}
\begin{enumerate}
\setlength{\itemsep}{1pt}
\item For every internal node $v$ for which $t$ is its leftmost descendant, allocate a new shuffle mechanism $M^{sum}_{v}$ for private summation with privacy parameters $(\eps/\log n,\delta/\log n)$ and shuffle size    
$|L_v|$, to the shuffler associated with the level of $v$.
\item The leaf node $t$ receives value $b_t$.
\item Let the leaf node $t$ participate in all the shuffle mechanisms $M^{sum}_{v}$ such that $t\in L_v$.\;
\item Execute all the mechanisms of the full shufflers, i.e., the mechanisms $M^{sum}_{v}$ such that $t$ is the rightmost descendant of $v$, and mark the corresponding shufflers inactive. For each  such $v$, we analyze the  output of the corresponding shuffler to get an estimate $S_v$ of the sum of values of $L_v$.\;
\item Output $\sum_{v\in V^{\star}_{t}}S_v$, where $V^{\star}_{t}$ is the set of at most $\log n$ internal nodes defined in Definition~\ref{def:vstartForBinaryTree}.\;
\end{enumerate}
\caption{\CSTA (Concurrent Shuffle Tree-Based Algorithm)}
\label{alg:summation}
\end{algorithm}

\looseness=-1
We bound the error of Algorithm~\ref{alg:summation} in Theorem~\ref{thm:treealgorithm}.
The proof is similar to the proof of Theorem 4.1 in \citet{dwork2010differential}.
It uses the fact that the additive error of $M_v^{sum}$ has zero mean and is 
sub-gaussian  with variance at most $V_{n,\eps,\delta}$ over batches of size at most $n$ for privacy parameters $\eps$ and $\delta$. 

\begin{theorem}~\label{thm:treealgorithm}
Consider a run of Algorithm~\ref{alg:summation} with parameters $n\in \naturals$ and $\eps,\delta > 0$.
At each time $t$ the error of the sum estimate is sub-gaussian with variance at most $O\left(V_{n,\frac{\eps}{\log n},\frac{\delta}{\log n}}\log n\right)$. Moreover, the algorithm is $(\eps,\delta)$-CSDP and an $(n,\alpha, \beta)$ summation algorithm for $\alpha = O\left(\sqrt{V_{n,\frac{\eps}{\log n},\frac{\delta}{\log n}}}\cdot (\sqrt{\log (1/\beta)\log n}+\log n)\right)$ and for any $\beta \in (0,1)$.
\end{theorem}

\looseness=-1
We apply Theorem~\ref{thm:treealgorithm} to two specific instances of private summation in the shuffle model. The first is for private binary summation, 
with Algorithm 3 of \citet{tenenbaum2021differentially} as the shuffle mechanism, denoted by $M^{sum}_{bin}$. For any $\eps<1$ and $\delta>0$, it has unbiased sub-gaussian error with variance $V_{n,\eps,\delta} = O\left(\frac{\log (1/\delta)}{\eps^2}\right)=\Tilde{O}\left(\frac{1}{\eps^2}\right)$. The second is for privately summing whole vectors and matrices with bounded $l_2$ norms, with the mechanism from Appendix C of \citep{chowdhury2022shuffle} as the shuffle mechanism, denoted by $M^{sum}_{vec}$.  It relies on an efficient and accurate mechanism of \citet{cheu2021shuffle}, and for any $\eps<15$ and $\delta<1/2$ it gets unbiased sub-gaussian error with variance $V_{n,\eps,\delta} = O\left(\frac{(\log (d/\delta))^2}{\eps^2}\right)=\Tilde{O}\left(\frac{1}{\eps^2}\right)$ for each entry separately.

\begin{corollary}~\label{thm:treealgorithmSummation}
Consider a run of Algorithm~\ref{alg:summation} on a binary input with parameters $n\in \naturals$, $\eps<1$, $\delta < 1/2$ and using $M^{sum}_{bin}$. At each time $t$, the error is sub-gaussian with variance at most $\Tilde{O}\left(\frac{1}{\eps^2}\right)$. Moreover, the algorithm is $(\eps,\delta)$-CSDP and an $(n,\alpha, \beta)$ summation algorithm for $\alpha = \Tilde{O}\left(\frac{1}{\eps}\right)$. The same holds for bounded vector inputs using $M^{sum}_{vec}$, for each entry separately.
\end{corollary}

\looseness=-1
From Figure~\ref{fig:treeBasedMechanismShufflers}, we can see that Algorithm~\ref{alg:summation} requires $k=\log n-1$ concurrent shufflers, one for each internal level in the tree. Note that we can optimize the memory of the algorithm to $O(\log n)$, by observing that there is no use in memorizing the sum estimate for two siblings, since the sum estimate of their parent replaces them and has lower error.

\subsection{Constrained number of shufflers}\label{sec:constrainedShufflers}
Recall that in order to get polylogarithmic error, Algorithm~\ref{alg:summation} 
used $k=\log n -1$ shufflers and
 every user participated in $k=\log n -1$ different shuffle mechanisms. A natural question is: ``What if we cannot support $\log n$ shufflers, but only $k \ll \log n$  shufflers?''. Can we still get better error guarantees than in the sequential shuffle model? 

We point out two natural modifications of Algorithm~\ref{alg:summation} to $k\ll \log n$ which give bad results. The first is to use only the $k$ lowest levels of the tree, i.e., shufflers of sizes $\{2^i\}_{i=1}^{k}$.  
This solution has a lager error since for many queries it has to aggregate results from  a large number of shuffle mechanisms each giving a noisy sum estimate with an independent noise. We call this kind of error an \textbf{inter-batch error}.
%
Specifically, the best estimate we get for the sum at each time is by summing the  estimates from all the previous roughly $n/2^k$ batches of size $2^k$, incurring an error proportional to $\sqrt{n/2^k}$.  (In general if we aggregate results from $s$ shuffle mechanisms we would get inter-batch error  $\Omega(\sqrt{s})$.)

The second natural modification is to take the $k$ highest levels of the tree, i.e., mechanisms of sizes $\{2^i\}_{i=\log n-k}^{\log n}$. Unfortunately, this solution is bad due to what we call a high \textbf{intra-batch error}. 
Specifically, since each mechanism is of size at least $2^{\log n-k}$, and its shuffler reports only when it fills up, the error after the $(2^{\log n-k}-1)$'th user can be $\Omega(2^{\log n-k})$ in the worst case.
In general, if the smallest shuffle mechanism is of size $b$ we can get $\Omega(b)$ intra-batch error.

To get small error for some  fixed $k\in \naturals$, 
we have to change the structure of the tree $T$ so that the number of children of any node is $d=\Theta(n^{\frac{2}{2k+1}})$ except for nodes at the lowest level (of internal nodes, which is level $1$) that are of degree $d_{low}=n/d^k=\Theta(n^{\frac{1}{2k+1}})$, to ensure that the tree has $d^kd_{low}=n$ leaves. To ease the presentation assume that $d_{low}$ and therefore $d$ are integers. Note that the root is not used since a shuffler for all users is useless. The major difference of this tree from the binary tree, 
is in the way we estimate the sum in each time.

As a warm-up, consider the sequential shuffle model in which there is only a single shuffler. i.e. $k=1$. The resulting tree has $n$ leaves, $d$ internal nodes over $d_{low}=n/d$ leaves each, and a single root above all the internal nodes. For this tree, the shuffler is run only for the internal level, and the intra-batch error is $d_{low}-1=O(n/d)$, whereas the inter-batch error at the last time $n$, scales like $\sqrt{d}$, since the best estimate for the sum is to sum the outputs of all the previous $d$ executed mechanisms. We balance the two error terms by setting $d=\Theta(n^{2/3})$ and $d_{low}=\Theta(n^{1/3})$. This gives an optimal algorithm with error term $\alpha=\Theta(n^{1/3})$. 

We extend this analysis to an arbitrary $k\ll \log n$. 
For each $t=1,\ldots,n$ we evaluate
the sum by adding up the estimates associated with the nodes of $V^\star_t$ as defined in Definition~\ref{def:vstartForBinaryTree} (which holds for a general tree).
%
For the inter-batch error, note that
each node of height $i=1,\ldots,k$ of the tree has at most $d-1$ siblings, so by the definition of $V^\star_t$, the height $i$ can contribute at most $d-1$ estimates. Therefore, each time $t$, we sum a total of at most $m=\sum_{i=1}^{k}(d-1)=k(d-1)=O\left(k\cdot n^{\frac{2}{2k+1}}\right)$
sum estimates.\jay{Is this fine? Wrong for the last node...} Since each user participates in at most $k$ mechanisms (matching the internal nodes that are its ancestors), it follows by simple composition that to preserve privacy, the internal nodes' mechanisms should be $(\eps/k,\delta/k)$-SDP. 
Assuming that these mechanisms have noise which is sub-gaussian with variance at most $V_{n, \frac{\eps}{k}, \frac{\delta}{k}}$ over batches of size at most $n$ and privacy parameters $\eps$ and $\delta$, we get that at each time $t$, the  error of the output is sub-gaussian with variance at most $O\left(V_{n, \frac{\eps}{k}, \frac{\delta}{k}}\cdot k\cdot n^{\frac{2}{2k+1}}\right)$. However, within the lowest level of internal nodes, the batch size is $d_{low}$, so the intra-batch error is at most $d_{low}-1=O\left(n^{\frac{1}{2k+1}}\right)$. Hence the total error in each time is sub-gaussian with variance at most $O\left(V_{n, \frac{\eps}{k}, \frac{\delta}{k}}\cdot k\cdot n^{\frac{2}{2k+1}}\right)$, and with expectation at most $O\left(n^{\frac{1}{2k+1}}\right)$.
The full algorithm appears in the Appendix (Algorithm~\ref{alg:summationv3})  for which we prove the following results.
\begin{theorem}~\label{thm:treealgorithmKShufflers}
For $n\in \naturals$, and $\eps,\delta>0$ and any number of shufflers $k$, for the algorithm above, in each time $t$ the sum estimate error is sub-gaussian with variance at most $O\left(V_{n,\frac{\eps}{k},\frac{\delta}{k}}\cdot k\cdot n^{\frac{2}{2k+1}} \right)$. Moreover the algorithm is $(\eps,\delta)$-CSDP and an $(n,\alpha, \beta)$ summation algorithm for $\alpha = O\left(\sqrt{V_{n,\frac{\eps}{k},\frac{\delta}{k}}}\cdot\sqrt{k}n^{\frac{1}{2k+1}}\cdot \sqrt{\log (1/\beta)+\log n}\right)$ and for any $\beta \in (0,1)$.
\end{theorem}

Analogously to Corollary~\ref{thm:treealgorithmSummation}, we conclude:
\begin{corollary}~\label{thm:treealgorithmSummationKShufflers}
For any number of shufflers $k\in \naturals$, the algorithm above run on a binary input with parameters $n\in \naturals$, $\eps<1$, $\delta < 1/2$ and using $M^{sum}_{bin}$, in each time $t$ the sum estimate error is sub-gaussian with variance at most $\Tilde{O}\left(\frac{k^3}{\eps^2}\cdot n^{\frac{2}{2k+1}}\right)$. Moreover, the algorithm is $(\eps,\delta)$-CSDP and an $(n,\alpha, \beta)$ summation algorithm for $\alpha = \Tilde{O}\left(\frac{k^{3/2}}{\eps}\cdot n^{\frac{1}{2k+1}}\right)$. The same holds for bounded vector inputs using $M^{sum}_{vec}$, for each entry separately.
\end{corollary}

\section{Lower bounds}\label{sec:lb}
We consider the specific problem of binary \psco{} where the inputs $b_t\in \{0,1\}$ are binary. In this section we show that all algorithms in the concurrent shuffle model which are $(n,\alpha,\beta)$ binary summation algorithms have  error $\alpha=\tilde{\Omega}(n^{1/(2k+1)})$. This specifically means that in the sequential shuffle model ($k=1$), all algorithms have error $\alpha=\tilde{\Omega}(n^{1/3})$, and that for any number $k$ of concurrent shufflers, our algorithm from Section~\ref{sec:constrainedShufflers} has optimal asymptotic dependence on $n$.

Our proof relies on some components of the $\Omega(\sqrt{n})$ lower  bound of \citet{vadhan2017complexity} for private summation in LDP. Specifically, we also use the fact that conditioned on any specific transcript $t$ of a DP communication protocol (between  parties with independent randomness and private data), the following two facts hold. (1) by privacy, the bit of each user has constant variance, and (2) these bits are independent (Lemma~\ref{lem:independence} in the Appendix).
%
As in \citet{vadhan2017complexity} 
these facts allow  us to use an anti-concentration bound (Lemma~\ref{lem:matouvsek} in the Appendix) which shows that the Hoeffding inequality is tight. Specifically, it shows that the sum of $m$ bounded independent random variables, each with variance $\sigma^2$, is \textbf{not} concentrated within any region of length $O(\sigma\sqrt{m})$ with high probability. 

To adapt this proof to our privacy model with $k$ concurrent shufflers we face multiple difficulties. For our proof, we cannot just use random bits as in \citet{vadhan2017complexity},  and we need to identify a hard and subtle family of distributions over user values, that allows us to use induction on $k$. We denote this family by $\Delta_{n,k}$. Let $c_\eps: = \frac{e^{\eps}}{e^{2\eps}+1}$. We define $\Delta_{n,k}$ to be the set of all distributions of binary allocations, which start with a sequence of at most $rep_{n,k}:=\frac{1}{2}n^{\frac{1}{2k+1}}\cdot c_\eps^{\frac{2k}{2k+1}}$ $0$'s or $1$'s, and continue with an alternation between a single $Ber(1/2)$ bit, and  $rep_{n,k}$ copies of another $Ber(1/2)$ bit, until length $n$ is reached.

We prove the theorem by induction on $k=0,1,\ldots$. For the induction step, we assume by contradiction the existence of an algorithm $M$ with error $o(rep_{n,k})$. We consider a distribution $\Delta\in \Delta_{n,k}$ and the resulting transcript random variable $T$ obtained by running $M$ on a vector $D$ of random user values  sampled from $\Delta$. Note that $T$ is determined by the sampled $D\sim \Delta$, and the internal randomness of the algorithm (in all shuffle mechanisms which were executed).

Then we split according to whether or not we allocated a mechanism over a large batch of size at least $Big_{n,k}:=n^{\frac{2k-1}{2k+1}}\cdot c_\eps^{\frac{2}{2k+1}}$  at any time. 
Intuitively, if we did, during the lifetime of this batch, we can use $M$ to get an algorithm for binary \psco{} within this  batch by computing differences of estimated sums of $M$. Since $M$ is assumed to be very accurate, the resulting mechanism contradicts the inductive argument for $Big_{n,k}$ users, $k-1$ shufflers and a distribution in $\Delta_{Big_{n,k},k-1}$. 

Assuming we didn't run a mechanism over a large batch (covered by Lemma~\ref{lem:NoConditionDprime} in the Appendix), let $Bits_{n,k}$ be the single (non-repeated) $Ber(1/2)$ bits of $D\sim \Delta$ which are evenly spaced at distance $Big_{n,k}$ of one another. For analysis, we consider an extended mechanism $M'$ which adds all the random bits of $D\sim \Delta$ which aren't of $Bits_{n,k}$ to the transcript as well, but returns the exact same value as $M$, hence has the same error. We show that the output of $M'$ is private with respect to the bits $Bits_{n,k}$ (Lemma~\ref{lem:privateOutput} in the Appendix), and this means (Lemma~\ref{lem:urioverdist} in the Appendix) that over the distribution $D\sim\Delta$, with high probability, conditioned on the transcript, each of the $O(n/Big_{n,k})$ bits in $Bits_{n,k}$ has constant variance. By defining a communication protocol which produces the transcript of $M'$, we conclude that the bits in $Bits_{n,k}$ are independent conditioned on the transcript (Lemma~\ref{lem:independence}).
Hence, by anti-concentration bounds (Lemma~\ref{lem:matouvsek}), we get an error lower bound of $\Tilde{\Omega}(\sqrt{\abs{Bits_{n,k}}})=\Tilde{\Omega}(\sqrt{n/Big_{n,k}})=\Tilde{\Omega}(rep_{n,k})$ in the sum approximation.

To split on the two cases above, we consider a partition of the transcript space, and apply the law of total probability to conclude the lower bound. 

\begin{theorem}\label{thm:lb}
Let $\eps$ and $\delta$ such that $\delta<\frac{\min(\eps,1)}{40n(k+1)}$, and let an $(\eps,\delta)$-CSDP $n$-user $k$-shuffler algorithm $M$ for binary \psco{} with adaptive mechanisms (encoders and sizes).  Consider a distribution $\Delta\in \Delta_{n,k}$ over user binary values. Then, $M$ has additive error $\alpha = \Omega(rep_{n,k})$ on some of its reported sum estimates with constant probability 
over the randomness of $D\sim \Delta$ and $M$.
\end{theorem}

\section{Contextual linear bandits}\label{sec:contextualLinearBandits}
In the linear contextual bandit problem \citep{auer2002using, chu2011contextual}, at each time $t=1,\ldots,n$ ($n$ is called the \textit{horizon}), a new user with an arbitrary (and possibly adversarial) private context $c_t\in C$ is recommended an action.
The server's goal is to recommend the user an action $a_t$ which maximizes the resulting expected reward. The reward function is characterized by the unknown parameter vector $\theta^* \in \reals^d$ and a known mapping $\phi:C\times A \to \reals^d$, from context-action pairs to a $d$-dimensional feature vector. Specifically, the reward is sampled as $\ip{\theta^*}{\phi(c_t,a_t)}+\eta_t$, where $\eta_t$ is a zero-mean sub-gaussian noise with variance at most a constant $\sigma^2$. The actions A and contexts C are arbitrary and can vary with time. We measure the utility of an algorithm for horizon $n$ by the cumulative pseudo-regret defined to be 
$$Reg(n)=\sum_{t=1}^{n}\left[max_{a\in A}\ip{\theta^*}{\phi(c_t,a)} - \ip{\theta^*}{\phi(c_t,a_t)}\right],$$
which quantifies the expected loss due to not knowing the optimal action at each time, since the parameter vector $\theta^*$ is unknown.

At each time  $t=1,\ldots,n$, the interaction between the server and the $t$'th user is according to the following steps:\\
(1) The server sends the user information (effectively an action selection rule) which deterministically maps any context $c$ to action $a$.\\
(2) The user applies the selection rule to its context $c_t$  to select the action $a_t\in A$ and gets the resulting vector $x_t=\phi(c_t,a_t)$.\\
(3) The user receives a stochastic reward $y_t=\ip{\theta^*}{x_t}+\eta_t$, where $\eta_t$ is zero-mean sub-gaussian noise with variance at most a constant $\sigma^2$. \\
(4) The user sends information to the server (in our case, encoded information of the vector $x_ty_t$ and matrix $x_tx_t^T$ through the shufflers to the server).\footnote{Since $a_t$ and $x_t$ are functions of the private information $c_t$, and $y_t$ is private information as well, they all include user-private information.}\\
 (5) The server updates the action selection rule.

The Linear Upper Confidence Bound (LinUCB) algorithm \citep{abbasi2011improved} is a well-studied  algorithm for the linear contextual bandits problem that has optimal regret. It maintains a confidence ellipsoid in which $\theta^*$ resides with high probability.
This ellipsoid at time $t$ is defined by the matrix $\sum_{s\leq t}x_sx_s^T$ and the vector $\sum_{s\leq t}x_sy_s$.
In LinUCB, the action decision rule sent to the users selects the action for each user optimisitically: It chooses the action which maximizes the inner product of the action with any point of the confidence ellipsoid. 

Contextual linear bandits are applicable in internet advertisement selection, recommendation systems and many more fields. This motivated a line of work \citep{shariff2018differentially, zheng2020locally, chowdhury2022shuffle, garcelon2022privacy} studying linear contextual bandit problems under the lens of differential privacy, to guarantee that the users’ private information cannot be inferred by an adversary during this learning process. Since LinUCB reduces the problem to accurately maintaining the matrix $\sum_{s\leq t}x_sx_s^T$ and the vector $\sum_{s\leq t}x_sy_s$, it suffices to ensure privacy with respect to the variables $x_t$ and $y_t$. \citet{shariff2018differentially} achieved this by injecting noise into these cumulative sums. 
Since the server only uses noisy versions of $\sum_{s\leq t}x_sx_s^T$ and 
$\sum_{s\leq t}x_sy_s$, they enlarged the confidence ellipsoid to ensure it contains $\theta^*$
with high probability.
Let $\rho_{max}$ be the largest magnitude of noise of any cumulative sum estimate in the sequence, then the adapted LinUCB algorithm of \citet{shariff2018differentially} ensures regret which scales like $\Tilde{O}(\sqrt{n}\cdot \sqrt{\rho_{max}})$.

For centralized DP, the error of a summation algorithm for horizon $n$ is $\rho_{max}=O(polylog(n))=\Tilde{O}(1)$, so \citet{shariff2018differentially} obtained a LinUCB-based centralized DP algorithm with regret $\Tilde{O}(\sqrt{n}\cdot \sqrt{1})=\Tilde{O}(\sqrt{n})$. 
For Local Differential Privacy (LDP), the error of a summation algorithm for horizon $n$ is $\rho_{max}=\Tilde{O}(\sqrt{n})$, so \citet{zheng2020locally} used the methodology of \citet{shariff2018differentially} to obtain a LinUCB-based LDP algorithm with regret $\Tilde{O}(\sqrt{n}\cdot \sqrt{\sqrt{n}})=\Tilde{O}(n^{3/4})$.
\citet{chowdhury2022shuffle} considered Shuffle Differential Privacy (SDP), grouping the users into sequential batches of constant size $B$. The error of their summation algorithm for horizon $n$ scales like $\rho_{max}=\Tilde{O}(\sqrt{n/B})$, so by extending the analysis of \citet{shariff2018differentially} to batches, they obtained an algorithm with regret $\Tilde{O}(B+\sqrt{n}\cdot \sqrt{\sqrt{n/B}})=\Tilde{O}(B+n^{3/4}/B^{1/4})$. 
Setting $B=n^{3/5}$, they got $\Tilde{O}(n^{3/5})$ regret and posed 
 the open question of whether regret $\Tilde{O}(\sqrt{n})$ can be achieved under any notion of privacy stronger than the centralized model. 

We give a positive answer to this open question in our Concurrent Shuffle DP (CSDP) model. Corollary~\ref{thm:treealgorithmSummation} 
gives a summation algorithm for horizon $n$ with error $\rho_{max}=\Tilde{O}\left(\frac{k^{3/2}}{\eps}\cdot n^{\frac{1}{2k+1}}\right)$ for $k$ concurrent shufflers. Hence using the framework of \citet{shariff2018differentially}, we get an algorithm for contextual linear bandits with regret which scales like $\Tilde{O}\left(\sqrt{n}\cdot \sqrt{\frac{k^{3/2}}{\eps}\cdot n^{\frac{1}{2k+1}}}\right)=\Tilde{O}\left(\frac{k^{3/4}}{\sqrt{\eps}}n^{\frac{k+1}{2k+1}}\right)$. This is formalized in the following theorem.
\begin{theorem}\label{thm:contextualLinearKShuffles}
    Fix a horizon $n\in \naturals$, a number of shufflers $k\in \naturals$, and privacy budgets $\eps<1$ and $\delta < 1/2$.  
    Then the algorithm as described above (and detailed in Appendix~\ref{sec:concShufflerLinUCBAppendix}) is $(\eps,\delta)$-CSDP and has regret $Reg(n)=\tilde{O}\left(\frac{k^{3/4}n^{\frac{k+1}{2k+1}}}{\sqrt{\eps}}\cdot (\sigma +d)\right)$.
\end{theorem}
In particular, for $k=\log n$ shufflers we get the following corollary.
\begin{corollary}\label{cor:contextualLinearLogShuffles}
    Fix a horizon $n\in \naturals$, and privacy budgets $\eps<1$ and $\delta < 1/2$.  
    Then the algorithm as described above (and detailed in Appendix~\ref{sec:concShufflerLinUCBAppendix}) using $k=\log n$  shufflers is $(\eps,\delta)$-CSDP and has regret $Reg(n)=\tilde{O}\left(\frac{\sqrt{n}}{\sqrt{\eps}}\cdot (\sigma +d)\right)$.\footnote{Actually, taking $k=\frac{1}{4}\log n -\frac{1}{2}$ is large enough to ensure the same asymptotic regret.} 
\end{corollary}

\section{Concluding remarks}
In this paper, we introduced and analyzed the novel concurrent shuffle model of differential privacy under continual observation. We demonstrated the effectiveness of the concurrent shuffle model in the problem of private summation, to get improved and tight algorithms, and in the contextual linear bandit problem, showing improved error bounds and similar privacy guarantees compared to the sequential shuffle model. 
Specifically, for $k=\log n$ shufflers, we use a variation of the tree-based algorithm of \citet{dwork2010differential} to get an algorithm for \psco{} with polylogarithmic error, and for contextual linear bandits, we close the gap of \citet{chowdhury2022shuffle} to the centralized model and achieve regret with asymptotic dependence on $n$ of $\Tilde{O}(\sqrt{n})$.

There are several directions for future work. One possibility is to explore other problems and applications where the concurrent shuffle model can be applied to achieve improved error bounds and strong privacy guarantees.

\section*{Disclosure of Funding}
This work is partially supported by Israel Science Foundation (grants 993/17,1595/19,1871/19), Len Blavatnik and the Blavatnik Family Foundation, the European Research Council (ERC) under the European Union’s Horizon 2020 research and innovation program (grant agreement 882396),  the Yandex Initiative for Machine Learning at Tel Aviv University.


\bibliography{references}
\bibliographystyle{icml2023}

\newpage
\appendix
\onecolumn
\section{Missing proofs from Section~\ref{sec:concurrentshuffler}}

\subsection{Missing proofs from Section~\ref{sec:concurrentshufflerPrivateSum}}

\begin{proof}[Proof of Theorem~\ref{thm:treealgorithm}]
Observe that the major differences between our algorithm and the interpretation of the tree-based algorithm of \citet{dwork2010differential} as presented in the Introduction (see Figure~\ref{fig:treeBasedMechanism}) are (1) our noise for each internal node is slightly different due to the shuffle mechanism, (2) since we do not allocate a shuffler for a single user (which would be wasteful), we have no internal node directly above each leaf so we produce the sum estimates using a different set of internal nodes.

The privacy is trivial by simple composition, since in each batch we use an $(\eps/\log n,\delta/\log n)$-SDP mechanism, and for each user there are at most $\log n$ mechanisms in which they participate, matching the internal nodes that are its ancestors.\footnote{For simplicity of presentation, we used simple composition instead of advance composition, which would shave another $\sqrt{\log n}$ from the final counter error $\alpha$ and another $\log n$ factor from the sub-gaussian variance in Corollary~\ref{thm:treealgorithmSummation}.}

To analyze the accuracy, observe that for each time $t$, we get a sum estimate by adding at most $\abs{V^\star_t}\leq \log n$ independent noises (one from each sum estimate), each one has zero mean and is sub-gaussian with variance $V_{n, \frac{\eps}{\log n}, \frac{\delta}{\log n}}$. Therefore, the sum of noises at each time $t$ is sub-gaussian with variance at most $
V_{n, \frac{\eps}{\log n}, \frac{\delta}{\log n}}\cdot \abs{V^\star_t}=V_{n, \frac{\eps}{\log n}, \frac{\delta}{\log n}}\log n$. By Hoeffding's inequality, with probability at least $1-\beta/n$, the magnitude of the sum of noises is at most $O\left(\sqrt{\log(1/\beta) + \log n} \cdot \sqrt{V_{n, \frac{\eps}{\log n}, \frac{\delta}{\log n}}\log n }\right) = O\left(\sqrt{V_{n, \frac{\eps}{\log n}, \frac{\delta}{\log n}}}\cdot (\sqrt{\log (1/\beta)\log n}+\log n)\right)$. 
By a union bound over the horizon $n$ of the algorithm, with probability at least $1-\beta$, in all $t=1,\ldots,n$ simultaneously, this sum of noises has magnitude $O\left(\sqrt{V_{n, \frac{\eps}{\log n}, \frac{\delta}{\log n}}}\cdot (\sqrt{\log (1/\beta)\log n}+\log n)\right)$.

For any even time $t$, since a shuffle mechanism has just ended at time $t$, the sum estimate is unbiased and the error is precisely the sum of noises bounded above. For any odd time $t$, we approximate its sum using the estimate of time $t-1$, so we accumulate an additional error $\abs{b_t}\leq 1$, which is dominated by the error above due to the sum of noises.
\end{proof}

\begin{proof}[Proof of Corollary~\ref{thm:treealgorithmSummation}]
    We first consider the binary summation setting.  \citet{tenenbaum2021differentially} showed that the error of $M^{sum}_{bin}$ on each single batch of size at most $n$ is sub-gaussian with variance at most $V_{n,\eps,\delta} = O\left(\frac{\log (1/\delta)}{\eps^2}\right)$, which for parameters $(n,\frac{\eps}{\log n},\frac{\delta}{\log n})$ gives variance $V_{n,\frac{\eps}{\log n},\frac{\delta}{\log n}}=O\left(\frac{\log (\log n/\delta)\log^{2} n}{\eps^2}\right)$.
    Applying Theorem~\ref{thm:treealgorithm} yields an $(n,\alpha, \beta)$ summation algorithm for $\alpha = O\left(\sqrt{V_{n, \frac{\eps}{\log n}, \frac{\delta}{\log n}}}\cdot (\sqrt{\log (1/\beta)\log n}+\log n)\right) = O\left(\sqrt{\frac{\log (\log n/\delta)\log^{2} n}{\eps^2}}\cdot (\sqrt{\log (1/\beta)\log n}+\log n)\right)=O\left(\frac{\sqrt{\log (\log n/\delta)}\left(\sqrt{\log (1/\beta)}\log^{1.5}n+\log^{2} n\right)}{\eps}\right)$, and in each time $t$ the sum estimate error is sub-gaussian with variance at most $O\left(\frac{\log (\log n/\delta)\log^{3} n}{\eps^2}\right)$.

    For the bounded vector setting, in the proof of Theorem C.1 of \citet{chowdhury2022shuffle}, they use the error bounds from Lemma 3.1 of \citet{cheu2021shuffle}, to 
    yield a mechanism $P_{vec}$, which we denote by $M^{sum}_{vec}$, which is $(\eps,\delta)$-SDP, and gives an unbiased estimate of the sum with noise which is sub-gaussian with variance $O\left(\frac{(\log (d/\delta))^2}{\eps^2}\right)$ within each entry.
    For parameters $(n,\frac{\eps}{\log n},\frac{\delta}{\log n})$, this gives variance $V_{n,\frac{\eps}{\log n},\frac{\delta}{\log n}}=O\left(\frac{(\log (d\log n /\delta))^2\log^2n}{\eps^2}\right)$.
    Applying Theorem~\ref{thm:treealgorithm} for each entry, this yields an $(n,\alpha, \beta)$ summation algorithm for $\alpha = O\left(\sqrt{V_{n, \frac{\eps}{\log n}, \frac{\delta}{\log n}}}\cdot (\sqrt{\log (1/\beta)\log n}+\log n)\right) = O\left(\sqrt{\frac{(\log (d\log n /\delta))^2\log^2n}{\eps^2}}\cdot (\sqrt{\log (1/\beta)\log n}+\log n)\right)=O\left(\frac{\log (d\log n/\delta)\cdot \left(\sqrt{\log (1/\beta)}\log ^{1.5}n + \log^{2} n\right)}{\eps}\right)$, and in each time $t$ the sum estimate error is sub-gaussian with variance at most $O\left(\frac{(\log (d\log n/\delta))^2}{\eps^2}\log^3 n\right)$ for each entry.
\end{proof}

\subsection{Missing parts from Section~\ref{sec:constrainedShufflers}}
We first give the missing algorithm from Section~\ref{sec:constrainedShufflers}, for \psco{} using $k\ll \log n$ concurrent shufflers, and then turn to the missing proofs. This algorithm is a modification of Algorithm~\ref{alg:summation}, using different degrees for different levels.
\begin{algorithm}[H]
\SetAlgoLined
\textbf{Tree.} Instantiate a tree with leaves as the numbers $1,\ldots,n$, where each level $i$ has degree $d=\Theta(n^{\frac{2}{2k+1}})$ except the lowest level (directly above the leaves) which has degree $d_{low}=\Theta(n^{\frac{1}{2k+1}})$.

\textbf{Internal Nodes.} Each internal node $v$ in the tree, is associated with the set of leaf nodes $L_v$ that are descendants of $v$.

\textbf{Processing.} For every time period $t\in \{1,\ldots,n\}$ corresponding to the $t$'th leaf in the tree:

\begin{enumerate}
\item For every internal node $v$ for which $t$ is its leftmost descendant, allocate a new shuffle mechanism $M^{sum}_{v}$ for private summation with privacy parameters $(\eps/k,\delta/k)$ and shuffle size $\abs{L_v}$, to the shuffler corresponding to the level of $v$.\;
\item The leaf node $t$ receives value $b_t$.\;
\item The leaf node $t$ participates in all the shuffle mechanisms $M^{sum}_{v}$ such that $t\in L_v$.\;
\item Execute all the mechanisms of the full shufflers, i.e., the mechanisms $M^{sum}_{v}$ such that $t$ is the rightmost descendant of $v$, and mark their corresponding shufflers inactive. For each such $v$, we analyze the shuffler output to get an estimate $S_v$ of the sum of values of $L_v$.\;
\item Output $\sum_{v\in V^{\star}_{t}}S_v$, where $V^{\star}_{t}$ is defined as in Definition~\ref{def:vstartForBinaryTree} (which holds for general trees).\;
\end{enumerate}
\caption{\CSTA (Concurrent Shuffle Tree-Based Algorithm for $k$ concurrent shufflers)}
\label{alg:summationv3}
\end{algorithm}

\begin{proof}[Proof of Theorem~\ref{thm:treealgorithmKShufflers}]
Observe that the difference from Theorem~\ref{thm:treealgorithm} is that now we have two sources of errors -- the inter-batch and intra-batch error.

The privacy is trivial by simple composition, since in each batch we use an $(\eps/k,\delta/k)$-SDP mechanism, and for each user there are at most $k$ mechanisms in which they participate, matching the internal nodes that are its ancestors.\footnote{For simplicity of presentation, we used simple composition instead of advance composition, which would shave another $\sqrt{k}$ from the final counter error $\alpha$ and another $k$ factor from the sub-gaussian variance in Corollary~\ref{thm:treealgorithmSummationKShufflers}.}

To analyze the accuracy, observe that for each time $t$ we get a sum estimate by adding at most $\abs{V^\star_t}=k(d-1)= O\left(k\cdot n^{\frac{2}{2k+1}}\right)$ independent noises (one from each sum estimate), each one has zero mean and is sub-gaussian with variance $V_{n,\frac{\eps}{k},\frac{\delta}{k}}$. This is since by the definition of $V^\star_t$, and since each node of height $1,\ldots,k$ of the tree has at most $d-1$ siblings, each such height $i$ can contribute at most $d-1$ internal node sums.\footnote{Note that an exception to this is the single last time $t=n$, in which $V^\star_t$ is the set of $d$ internal nodes which are children of the root. However, for this case $V^\star_t$ is still of size $d=O(dk)$, so the argument holds for it too.} Hence, similarly to the proof of Theorem~\ref{thm:treealgorithm}, the sum of noises at time $t$ is sub-gaussian with zero mean variance at most $V_{n,\frac{\eps}{k},\frac{\delta}{k}}\cdot s=O(V_{n,\frac{\eps}{k},\frac{\delta}{k}}\cdot k\cdot n^{\frac{2}{2k+1}})$, and with probability at least $1-\beta$, in all $t=1,\ldots,n$ simultaneously, the sum approximation error has expectation $O\left(\sqrt{\log(1/\beta) + \log n} \cdot \sqrt{V_{n,\frac{\eps}{k},\frac{\delta}{k}}\cdot (1+(k-1)(d-1)) } \right)= O\left(\sqrt{V_{n,\frac{\eps}{k},\frac{\delta}{k}}}\cdot \sqrt{1+k\cdot n^{\frac{2}{2k+1}}}\cdot \sqrt{\log (1/\beta)+\log n}\right)=O\left(\sqrt{V_{n,\frac{\eps}{k},\frac{\delta}{k}}}\cdot\sqrt{k}n^{\frac{1}{2k+1}}\cdot \sqrt{\log (1/\beta)+\log n}\right)$.

For any time $t$ which is divisible by $d_{low}$, since a shuffle mechanism has just ended at time $t$, the sum estimate is unbiased and the error is precisely the sum of noises bounded above. For any other time $t$ which is not divisible by $d_{low}$, we approximate its sum using the same sum estimate from the largest multiple of $d_{low}$ smaller than $t$. Hence, we accumulate an additional error of the sum of all the values $b_i$ for $i$ ranging from the last multiple of $d_{low}$ to $t$. Since each value $b_i$ is bounded in $[0,1]$, this error is at most $d_{low}-1=O\left(n^{\frac{1}{2k+1}}\right)$ which is dominated by the error above due to the sum of noises.
\end{proof}

\begin{proof}[Proof of Corollary~\ref{thm:treealgorithmSummationKShufflers}]
    The proof follows from an identical argument as the proof of Corollary~\ref{thm:treealgorithmSummation}, where for binary summation we get $V_{n,\frac{\eps}{k},\frac{\delta}{k}}=O\left(\frac{\log (k/\delta) k^2}{\eps^2}\right)$, so applying Theorem~\ref{thm:treealgorithmKShufflers} gives that the error in each time is sub-gaussian with variance at most $O\left(\frac{\log (k/\delta) k^2}{\eps^2}\cdot k\cdot n^{\frac{2}{2k+1}} \right)$, and the algorithm is an $(n,\alpha, \beta)$ summation algorithm for $\alpha = O\left(\sqrt{\frac{\log (k/\delta) k^2}{\eps^2}}\cdot\sqrt{k}n^{\frac{1}{2k+1}}\cdot \sqrt{\log (1/\beta)+\log n}\right)$.
    
    For vector summation we get $V_{n,\frac{\eps}{k},\frac{\delta}{k}}=O\left(\frac{\log^2 (dk/\delta) k^2}{\eps^2}\right)$, so applying Theorem~\ref{thm:treealgorithmKShufflers} gives that the error in each time is sub-gaussian with variance at most $O\left(\frac{\log^2 (dk/\delta) k^2}{\eps^2}\cdot k\cdot n^{\frac{2}{2k+1}} \right)$, and the algorithm is an $(n,\alpha, \beta)$ summation algorithm for $\alpha = O\left(\sqrt{\frac{\log^2 (dk/\delta) k^2}{\eps^2}}\cdot\sqrt{k}n^{\frac{1}{2k+1}}\cdot \sqrt{\log (1/\beta)+\log n}\right)$.
\end{proof}

\section{Missing parts from Section~\ref{sec:lb}}
In this section, we prove a lower bound for \psco{} (specifically for binary values), for any number $k$ of shufflers, even for mechanisms that can choose the shuffle mechanisms (encoders and sizes) adaptively.
The proof is by induction on the concurrency level $k$ and is based upon two subtle ideas: First, identifying the hard distribution on inputs that would make any mechanism fail. Second, a clever partition of the possible transcripts of the algorithm that allows us to show that whatever batch sizes the algorithm chooses it accumulates a large error with high probability. 
Specifically we show that if the algorithm uses a mechanism with a large batch then  while the shuffler running the mechanism is active the concurrency level reduces and the large error follows by induction. Otherwise, all batches are small. This implies that many input bits are independent even given the transcript of the algorithm (see Lemma~\ref{lem:independence}). Furthermore, these bits also have large variance (see Lemma~\ref{lem:urioverdist}) due to privacy so an anti-concentration bound (see Lemma~\ref{lem:matouvsek}) implies that we must accumulate a large error.

We now present and prove the intermediate general results for the proof presented above. 
First, we prove Lemma~\ref{lem:independence} which shows that conditioned on the transcript, input bits which were independent continue to be independent. 
\begin{lemma}[Independence given transcript, based on Lemma 2 of \citet{chan2012optimal}]\label{lem:independence}
    Let a set of $n$ parties participating in a communication protocol where the $i$'th party depends on a random variable  $X_i$ that he gets as input. The parties communicate publicly, where in each round a single party publishes a value depending on their input data, their internal private randomness, and the previous messages. 
    Then if $X_1,\ldots,X_n$ are independent, then they are still independent, even conditioned on all the communicated messages of the protocol.
\end{lemma}
\begin{proof}[Proof of Lemma~\ref{lem:independence}]

 We prove the  statement by induction on the number of rounds of messages. For the base case, after zero rounds the values are independent since the communication protocol didn't start yet. For the step, consider the $j$'th message $m_j$ sent by the party who has input $X_i$, and suppose $X_{-i}$ is the joint input of all other parties and $m_{<j}$ is the vector of all the previous messages $(m_1,\ldots,m_{j-1})$. Observe that by the induction hypothesis, $X_i$ is independent of $X_{-i}$ given $m_{<j}$, and therefore, since $m_j$ is a function of $X_i$, the randomness of the $i$'th user and $m_{<j}$, we conclude that $(X_i,m_j)$ is independent of $X_{-i}$ given $m_{<j}$.  We have that, 
 \begin{align*}
  \Pr[X_i=a,~X_{-i}=b \mid m_j=c,~m_{<j}=d]&=\frac{\Pr[X_i=a,~X_{-i}=b,~ m_j=c \mid m_{<j}=d]}{\Pr[m_j=c\mid m_{<j}=d]}\\
  &=\frac{\Pr[X_i=a,~ m_j=c \mid m_{<j}=d]\Pr[X_{-i}=b\mid m_{<j}=d]}{\Pr[m_j=c\mid m_{<j}=d]}\\
  &=\frac{\Pr[X_i=a,~ m_j=c \mid m_{<j}=d]}{\Pr[m_j=c\mid m_{<j}=d]}\cdot \Pr[X_{-i}=b\mid m_{<j}=d]\\
  &=\Pr[X_i=a\mid m_j=c,~m_{<j}=d]\cdot \Pr[X_{-i}=b \mid m_j=c,~m_{<j}=d],
 \end{align*}
 where the first step follows by the definition of conditional probability, the second step follows by the observation above, and the last step follows by the definition of conditional probability and since $m_j$ is independent of $X_{-i}$ given $m_{<j}$. We conclude that $X_i$ and $X_{-i}$ are independent conditioning also on $m_j$.
\end{proof}

Next, we prove Lemma~\ref{lem:urioverdist} which shows that if a mechanism is $(\eps,\delta)$-DP, then over a distribution of inputs and changing a specific bit consistently, has a bounded influence over the output distribution. As a warm-up we first present the following folklore lemma. It says that of a mechanism is $(\eps,\delta)$-DP then on a large subset of the outputs it is pure $2\eps$ private.
\begin{lemma}\label{lem:uri}
    Let $M:X\to Y$ be an $(\eps,\delta)$-DP mechanism. Then for any two neighboring datasets $D,D'\subseteq X$, there is a subset $E(D,D')\subseteq Y$ such that $\Pr[M(D)\in E],\Pr[M(D')\in E]\geq 1-O(\frac{\delta}{\eps})$ and $\forall t\in E$, $\frac{\Pr_M[M(D)=t]}{\Pr_M[M(D')=t ]}\in [e^{-2\eps},e^{2\eps}]$.
\end{lemma}
\begin{proof}[Proof of Lemma~\ref{lem:uri}]
    Let $D,D'$ be neighboring datasets, and consider the set of outputs $Bad_1=\{y\in Y \mid \frac{\Pr_M[M(D)=y]}{\Pr_M[M(D')=y]}<\exp(-2\eps)\}\subseteq Y$ and $Bad_2=\{y\in Y \mid \frac{\Pr_M[M(D)=y]}{\Pr_M[M(D')=y]}>\exp(2\eps)\}\subseteq Y$. By approximate-DP, $\Pr_M[M(D')\in Bad_1]\leq e^\eps\cdot \Pr_M[M(D)\in Bad_1]+\delta$ and 
    therefore $\Pr_M[M(D)\in Bad_1]=\sum_{y\in Bad_1}\Pr_M[M(D)=y]<\sum_{y\in Bad_1}exp(-2\eps)\Pr_M[M(D')=y]=\exp(-2\eps)\Pr_M[M(D')\in Bad_1]\leq \exp(-2\eps)\cdot \left(\exp(\eps)\cdot \Pr_M[M(D)\in Bad_1]+\delta\right)=\exp(-\eps)\cdot \Pr_M[M(D)\in Bad_1]+\delta \exp(-2\eps)$. Thus, $\Pr_M[M(D)\in Bad_1]<\frac{\delta \exp(-2\eps)}{1-\exp(-\eps)}$. A similar analysis shows that $\Pr_M[M(D)\in Bad_2]<\frac{\delta \exp(\eps)}{\exp(\eps)-1}$.

    Let $Bad=Bad_1\cup Bad_2$. Since $Bad_1$ and $Bad_2$ are disjoint, we have that $\Pr_M[M(D)\in Bad] = \Pr_M[M(D)\in Bad_1]+\Pr_M[M(D)\in Bad_2] < \frac{2\delta \exp(\eps)}{1-\exp(-\eps)} = O(\delta/\eps)$. From approximate-DP, we therefore get that $\Pr_M[M(D')\in Bad]\leq \exp(\eps)\cdot \frac{2\delta \exp(\eps)}{1-\exp(-\eps)} + \delta = O(\delta/\eps)$.

    Let $E(D,D') = Y\setminus Bad$. It follows that $\Pr_M[M(D)\in E(D,D')]>1-O(\delta/\eps)$, $\Pr_M[M(D')\in E(D,D')]>1-O(\delta/\eps)$, and most importantly by the definition of $Bad$, for any $t\in E(D,D')$ it holds that $t\notin Bad_1$ and $t\notin Bad_1$, and therefore $\frac{\Pr_M[M(D)=t]}{\Pr_M[M(D')=t ]}\in [e^{-2\eps},e^{2\eps}]$.    
\end{proof}

We now extend Lemma~\ref{lem:uri} to distributions.

\begin{lemma}[Extended over distribution]\label{lem:urioverdist}
    Let $M:X^n\to Y$ be an $(\eps,\delta)$-DP mechanism, and a distribution $A$ of datasets of size $n$. For $i \in [n]$, let $D_{i\to a}$ denote the dataset $D$ with the $i$'th element set to $a$. Then, for any two distinct elements $x_1,x_2\in X$, and index $i$, there is a subset $E(i,x_1,x_2,A)\subseteq Y$ such that:
    \begin{enumerate}
        \item $\Pr_{M,D\sim A}[M(D_{i\to x_1})\in E],\Pr_{M,D\sim A}[M(D_{i\to x_2})\in E]\geq 1-\frac{2\delta}{\min(\eps,1)}$, and
        \item $\forall t\in E$, $\frac{\Pr_{M,D\sim A}[M(D_{i\to x_1})=t]}{\Pr_{M,D\sim A}[M(D_{i\to x_2})=t]}\in [e^{-2\eps},e^{2\eps}]$.
    \end{enumerate}
\end{lemma}
\begin{proof}[Proof of Lemma~\ref{lem:urioverdist}]
    Let $i\in [n]$ and $x_1,x_2\in X$. consider the set of outputs $Bad_1=\{y\in Y \mid \frac{\Pr_{M,D\sim A}[M(D_{i\to x_1})=y]}{\Pr_{M,D\sim A}[M(D_{i\to x_2})=y]}<\exp(-2\eps)\}\subseteq Y$ and $Bad_2=\{y\in Y \mid \frac{\Pr_{M,D\sim A}[M(D_{i\to x_1})=y]}{\Pr_{M,D\sim A}[M(D_{i\to x_2})=y]}>\exp(2\eps)\}\subseteq Y$. Observe that
    \begin{align*}
        \Pr_{M,D\sim A}[M(D_{i\to x_1})\in Bad_1]&=\sum_{y\in Bad_1}\Pr_{M,D\sim A}[M(D_{i\to x_1})=y]\\
        &<\sum_{y\in Bad_1}\exp(-2\eps)\Pr_{M,D\sim A}[M(D_{i\to x_2})=y]\\
        &= \exp(-2\eps)\Pr_{M,D\sim A}[M(D_{i\to x_2})\in Bad_1]\\
        &= \exp(-2\eps)\sum_{D\in X^n}\Pr_{M}[M(D_{i\to x_2})\in Bad_1]\cdot \Pr_A[D]\\
        &\leq \exp(-2\eps)\sum_{D\in X^n} \left(\exp(\eps)\cdot \Pr_M[M(D_{i\to x_1})\in Bad_1]+\delta\right)\cdot \Pr_A[D]\\
        &= \exp(-2\eps)\delta+\exp(-\eps)\cdot \sum_{D\in X^n}\Pr_M[M(D_{i\to x_1})\in Bad_1]\cdot \Pr_A[D]\\
        &= \exp(-2\eps)\delta+\exp(-\eps)\cdot \Pr_{M,D\sim A}[M(D_{i\to x_1})\in Bad_1],
    \end{align*}
    where the first step follows by summing over $Bad_1$, the second step follows by the definition of $Bad_1$, the fourth step follows by the law of total probability over $D\sim A$, the fifth step follows by approximate-DP, where for any dataset $D \in X^n$, the datasets $D_{i\to x_1}$ and $D_{i\to x_2}$ are neighbors, and the seventh step follows by the law of total probability over $D\sim A$.

    Thus, $\Pr_{M,D\sim A}[M(D_{i\to x_1})\in Bad_1]<\frac{\delta \exp(-2\eps)}{1-\exp(-\eps)}$. A similar analysis shows that $\Pr_{M,D\sim A}[M(D_{i\to x_1})\in Bad_2]<\frac{\delta \exp(\eps)}{\exp(\eps)-1}$.

    Let $Bad=Bad_1\cup Bad_2$. Since $Bad_1$ and $Bad_2$ are disjoint, we have that 
    \begin{align*}
        \Pr_{M,D\sim A}[M(D_{i\to x_1})\in Bad] &= \Pr_{M,D\sim A}[M(D_{i\to x_1})\in Bad_1]+\Pr_{M,D\sim A}[M(D_{i\to x_1})\in Bad_2]\\
        &< \frac{\delta \exp(-2\eps)}{1-\exp(-\eps)} + \frac{\delta \exp(\eps)}{\exp(\eps)-1}\leq \frac{2\delta}{\min(\eps,1)}
    \end{align*}

    A symmetric defining argument shows that $\Pr_{M,D\sim A}[M(D_{i\to x_2})\in Bad] \leq \frac{2\delta}{\min(\eps,1)}$.

    Let $E(i,x_1,x_2,A) = Y\setminus Bad$. Therefore, $\Pr_{M,D\sim A}[M(D_{i\to x_1})\in E(i,x_1,x_2,A)]>1-\frac{2\delta}{\min(\eps,1)}$, $\Pr_{M,D\sim A}[M(D_{i\to x_2})\in E(i,x_1,x_2,A)]>1-\frac{2\delta}{\min(\eps,1)}$, and most importantly by the definition of $Bad$, for any $t\in E(i,x_1,x_2,A)$ it holds that $t\notin Bad_1$ and $t\notin Bad_2$, and therefore $\frac{\Pr_{M,D\sim A}[M(D_{i\to x_1})=t]}{\Pr_{M,D\sim A}[M(D_{i\to x_2})=t ]}\in [e^{-2\eps},e^{2\eps}]$.
\end{proof}

Lemma \ref{lem:matouvsek} specifies the anti-concentration bound which we use.
\begin{lemma}[\citet{matouvsek2001lower}]\label{lem:matouvsek}
There exist constants $a,b,c>0$ s.t.\ the following holds. Let $X$ be a sum of independent random variables, each attaining values in $[0,1]$, and assume $Var[X]\geq 200^2$. Then for all $a\geq \beta\geq2^{-bVar[X]}$ and every interval $I\in \reals$ of length $\abs{I}\leq c\cdot \sqrt{Var[X]\cdot \log (\frac{a}{\beta})}$ we have $\Pr[X\notin I]\geq \beta$.
\end{lemma}

The last general result we need is Lemma~\ref{lem:privateOutput}, which shows that an algorithm which outputs the result of an $(\eps,\delta)$-DP algorithm together with a constant subset of the input bits $I\subseteq [n]$, is private with respect to the bits which it does not output.
\begin{lemma}\label{lem:privateOutput}
    Let $M:X^n\to Y$ be an $(\eps,\delta)$-DP algorithm over $n$ users, and let  $I\subseteq [n]$ be a subset of users. Consider the randomized algorithm $M':X^n\to Y\times X^{\abs{I}}$, which given $D=(x_1,\ldots,x_n)$, computes the random output $(M(D),D^I)$ where $D^I:=(x_i)_{i\in I}$ is the concatenation of the values of the users in  $I$. Then $M'$ is $(\eps,\delta)$-DP with respect to the inputs $D^{-I}:=(x_i)_{i\notin I}$.
\end{lemma}
\begin{proof}[Proof of Lemma~\ref{lem:privateOutput}]
    We identify each dataset $D=(x_1,\ldots,x_n)$ by $D=(D^{I},D^{-I})$, separating  the values of users in $I$ and not in $I$.
   Consider a pair of neighbors $D=(D^{I},D^{-I})$, $D'=(D^{I},(D')^{-I})$ of $n$ users for $M'$, where $D^{-I}$ and $(D')^{-I}$ differ at  a single index, and let $B\subseteq Y\times X^{\abs{I}}$ be an output subset. It remains to show that $\Pr[M'(D)\in B]\leq e^\eps\cdot \Pr[M'(D')\in B]+\delta$.

    Indeed, let $B_{D^{I}}:=\{r\in Y\mid (r,D^{I})\in B\}\subseteq Y$. We have that
    \begin{align*}
        \Pr[M'(D)\in B]&=\Pr[M'(D)\in B_{D^{I}}\times \{D^{I}\}]=\Pr[M(D)\in B_{D^{I}}]\\
        &\leq e^\eps\cdot \Pr[M(D')\in B_{D^{I}}]+\delta = e^\eps\cdot \Pr[M'(D')\in B_{D^{I}}\times \{D^{I}\}]+\delta\\
        &=e^\eps\cdot  \Pr[M'(D')\in B] +\delta,
    \end{align*}
    where the first and fifth steps follow since the second output of $M'$ is dataset's bits at indices $I$, and since the bits at indices $I$ of both $D$ and $D'$ are equal $D^{I}$, the second and fourth steps follow since the bits at indices $I$ of both $D$ and $D'$ are equal $D^{I}$, and the third step follows since $M$ is $(\eps,\delta)$-DP and since $B_{D^{I}}\subseteq Y$.
\end{proof}

We now prove the lower bound, which is the main result of this section. This proof relies also on Lemma~\ref{lem:NoConditionDprime}, however since this lemma requires the  context and definitions of the proof, we state and prove it following the proof of Theorem~\ref{thm:lb} below.
We repeat some of the notation from section~\ref{sec:lb} for convenience: $c_\eps:=\frac{e^\eps}{e^{2\eps}+1}$, $rep_{n,k}:=\frac{1}{2}n^{\frac{1}{2k+1}}\cdot c_\eps^{\frac{2k}{2k+1}}$ and $Big_{n,k}:=n^{\frac{2k-1}{2k+1}}\cdot c_\eps^{\frac{2}{2k+1}}$. Let $\Delta_{n,k}$ be the set of all distributions of binary allocations, which start with a sequence of at most $rep_{n,k}$ $0$'s or $1$'s, and continue with an alternation between a single $Ber(1/2)$ bit, and  $rep_{n,k}$ copies of another $Ber(1/2)$ bit, until length $n$ is reached. We also introduce new notation.\footnote{All the $\Delta\in \Delta_{n,k}$ are in fact the same distribution, but the difference is only in the identity of the prefix of at most $rep_{n,k}$ $0$'s or $1$'s, that is, the length of this prefix and its bits.} Let $Bits_{n,k}:=\{X_i\}_{i=1}^{m}$ be the single (non-repeated) $Ber(1/2)$ bits of $D\sim \Delta\in \Delta_{n,k}$ which are evenly spaced at distance $Big_{n,k}$ of one another. Note that since there are at most $Big_{n,k}$ bits from the beginning of the input to $X_1$, it holds that $m\geq n/Big_{n,k}-1$.

\begin{proof}[Proof of Theorem~\ref{thm:lb}]
We prove the theorem by induction over the number of shufflers $k=0,1,2,\ldots$. Specifically, we show that given $\delta<\frac{\min(\eps,1)}{40n(k+1)}$, a large enough
$n$ such that $nc_{\eps}^2/Big_{n,k}>200^2$, an $(\eps,\delta)$-CSDP $n$-user $k$-shuffler algorithm $M$ for binary \psco{} with adaptive mechanisms (encoders and sizes), and a distribution $\Delta\in \Delta_{n,k}$ over user binary values, $M$ has additive error $\alpha = \Omega(rep_{n,k})$ on some of its reported sum estimates with probability $P_{n,k}=\frac{\beta_0}{2}\cdot \left(1-\frac{2n\delta k}{\min(\eps,1)}\right)$
over the randomness of $D\sim \Delta$ and $M$.

For the base case $k=0$, we have that $rep_{n,0}=n/2$ and $\Delta_{n,0}$ is the set of distributions of $n/2$ constant bits, then a single $Ber(1/2)$ bit and then $n/2-1$ copies of the same $Ber(1/2)$ bit. For any $\Delta\in \Delta_{n,0}$ for $k=0$ shufflers, while there are up to $n/2$ constant bits, there is still a sequence of at least $n-n/2-1$ bits which can be all zeros or all ones, i.e., there are equal probability pairs of sequences in the support whose sum differ by at least $n/2-1$. Since the outcome of the algorithm is deterministic (the server is deterministic, and no data is sent to it), and since , the error of the algorithm is at least $n/4-1=\Omega(rep_{n,0})$.

For the step, let $\eps$ and $\delta$ be the privacy parameters, and let $(\eps,\delta)$-CSDP $n$-user $k$-shuffler algorithm $M$ for binary \psco{} with adaptive mechanisms. 
 
Observe that our sources of randomness are the random vector $D$ and the randomness $S=(S_i)_{i=1}^{\#batches}$ of the encoders and the shuffles performed within the shuffle mechanisms by all $k$ shufflers. Note that given a concrete $D$ and $S$, since we assume that the server algorithm is deterministic, the random variable $T = T (D, S)$ denoting the transcript of the algorithm as sent to the server is deterministically defined.\footnote{The transcript is the shuffled encoded bits throughout all the batches, ordered by the termination time of the batch, and then by an arbitrary but consistent order over the shufflers.} 

Our proof follows by partitioning the set of possible transcripts, and showing that each set of transcripts either occurs with low probability, or that conditioned on the set, $M$ has large additive error. The additive error will follow by the differential privacy of $M$ with respect to $Bits_{n,k}$. To achieve this, we would like to condition the probability space on the values of the bits of $D$ which are not $Bits_{n,k}$ (including both the constant bits of the beginning of the sequences of $D$ and the random bits) denoted by $\overline{Bits_{n,k}}$. For technical reasons, we achieve this conditioning through an analysis of an alternative algorithm $M'$. Given an input $D\sim \Delta$ and randomness $S$ for $M$, $M'$ simulates $M$ using the same randomness $S$, and outputs the approximation of $M$ for the sum at each time. However, $M'$ produces the more detailed transcript $T'$ which includes both the transcript $T$ from the simulation of $M(D)$, and $\overline{Bits_{n,k}}$. This transcript $T'$ is created similarly to the transcript of $M$, but after each time $t$, before including the output of the shufflers of $M$ which were closed at $t$, it outputs the binary value of each new random bit (or constant bit in the beginning of $D$) of $\overline{Bits_{n,k}}$ when it is first observed.


We now define a partition of the space of transcripts  of $M'$. Note that from $T'(D,S)$ we can infer the number of batches of each shuffler and their sizes, but we cannot infer the dataset $D$ or the randomness $S$. To define the partition, we define two subsets of transcripts of $M'$. Let $E'$ be the subset of transcripts $t'$ of $M'$ in which for all $i=1,\ldots,m$ it holds that $\frac{\Pr_{D,S}[T'= t' \mid X_i=1]}{\Pr_{D,S}[T'= t' \mid X_i=0]}\in [e^{-2\eps},e^{2\eps}]$. Let $T'_{small,n,k}$ be the subset of transcripts of $M'$ where all batches throughout all shufflers are of size $\leq Big_{n,k}$. We partition the transcripts of $M'$ below by grouping by the prefix of the transcript until right before the first batch of size $> Big_{n,k}$, 
and if there was no such batch, it groups based on whether or not the transcript is contained in $ E'$. Specifically, the partition is:
 \begin{enumerate}
    \item The (unique) set intersection $T'_{small,n,k}\cap \overline{E'}$.
     \item The (unique) set intersection $T'_{small,n,k}\cap E'$. 
     \item For each specific transcript prefix $p$ of a run until it includes its first batch of size $> Big_{n,k}$ in any of its shufflers, consider the set $T_{p}$ of all $p$'s transcript completions.
 \end{enumerate}
We now bound the error for each set separately.

\textbf{The set of type (1):} We observe that this set is of small probability. Indeed, by applying Lemma~\ref{lem:privateOutput} to $M$ with respect to $\overline{Bits_{n,k}}$ we conclude that the computation of $T'$ is $(\eps,\delta)$-DP with respect to the bits $Bits_{n,k}$. For each $i$, we now apply Lemma~\ref{lem:urioverdist} (which we can apply, since the bits $X_i$ are independent, and hence sampling from $\Delta$ conditioned on an $X_i$ is equivalent to sampling from $\Delta$ and overriding $X_i$) to $M'$, with $A$ taken to be the distribution $\Delta$, $Y$ is the transcript $T'$, $i$ refers to the $i$'th bit of $Bits_{n,k}$ (recall that this is $X_i$), and $x_1=0$ and $x_2=1$ determine the binary value of this bit. We get that there exists a subset $E'_i$ of transcripts of $M'$ of probability $\Pr_{D,S}[T'\in E'_i \mid X_i=1],\Pr_{D,S}[T'\in E'_i \mid X_i=0] \geq 1-\frac{2\delta}{\min(\eps,1)}$ such that $\forall t'\in E'_i$, $\frac{\Pr_{D,S}[T'=t' \mid X_i=1]}{\Pr_{D,S}[T'=t' \mid X_i=0]}\in [e^{-2\eps},e^{2\eps}]$. By a union bound, it holds that $\Pr_{D,S}[T'\in E'] \geq 1-\frac{2n\delta}{\min(\eps,1)}$. 
Therefore, $\Pr_{D,S}[T'\in T'_{small,n,k}\cap \overline{E'}]\leq \Pr_{D,S}[T'\in \overline{E'}] \leq \frac{2n\delta}{\min(\eps,1)}$.

\textbf{The set of type (2):}
For each transcript $t'\in T'_{small,n,k}\cap E'$, Lemma~\ref{lem:NoConditionDprime} which is formulated below, implies that conditioned on $T'=t'$, $M'$ has additive error $\alpha = \Omega(rep_{n,k})$ at some time, with  probability $\beta_0 / 2$, over the randomness of $D\sim \Delta$ and $M'$.

\textbf{Sets of type (3):}
Let $p$ be a prefix of a transcript of a run until right before
time $\tau$ when a large batch $B$ of size $\abs{B}>Big_{n,k}$ starts. Assume without loss of generality that $B$ is processed by  shuffler number $k$, and that $\abs{B}=Big_{n,k}$ (otherwise we just look at a prefix of $B$ of size $Big_{n,k}$).
The rest of the argument is in the subspace of $D$ and $S$ conditioned on $p$.
We claim that in this subspace $M'$ has error $\Omega\left(rep_{n,k}\right)$  at some time $t$ with probability $\geq P_{n,k-1}$. 

To show this, we assume by contradiction that $M'$ has error $O\left(rep_{n,k}\right)$ in all its sum estimates with probability $> 1-P_{n,k-1}$ and we will get a contradiction. 
Specifically, by our assumption, there exists a fixed randomness $S_p$ for  shuffle mechanisms that end before time $\tau$, and
fixed sizes for shuffle mechanisms that overlap time $\tau$ (we do not fix the randomness of the encoders of these mechanisms that overlap time $\tau$), and a fix 
 set of binary values $D_p$ of the users arriving before $\tau$, conditioned on which $M'$ has error $O\left(rep_{n,k}\right)$ in all  prefixes with probability $> 1-P_{n,k-1}$.

We now define a $(k-1)$-shuffler algorithm $M'_B$ for binary \psco{} over $B$, which uses $M'$ internally, and we show that since $M'$ has low error conditioned on $S_p$ and $D_p$, the error of $M'_B$  on $B$ is too low and  contradicts the induction hypothesis.
Specifically, the algorithm $M'_B$ runs on $\abs{B}$ users arriving as batch $B$ using the following steps.
\begin{enumerate}
    \item It simulates the shuffle mechanisms on all batches that end before time $\tau$ on the server using $D_p$ and $S_p$. Then it simulates the shuffle mechanisms on batches overlapping $\tau$ using the shufflers $1,\ldots,k-1$
    by injecting to each of them the encoding (using fresh randomness of the encoders)
    of the appropriate suffix of $D_p$. The sizes of these $k-1$ mechanism is determined by $S_p$.
    \item It runs on the actual inputs that arrive in batch $B$ using the randomness of $M'$.
  At each time $j$, during batch $B$, $M'_B$ sets its sum estimate (of the prefix of $B$) to be the difference between the estimated  sum of $M'$ at time $j$ minus the sum of the bits of $D_p$.
\end{enumerate}
 Note that during the $\abs{B}-1$ different times within $B$, while simulating $M'$, the server does not receive any update from the $k$'th shuffler since it is still running. Therefore $M'_B$  uses only $(k-1)$ shufflers to get its estimates.



Let $\Delta_B$ be the distribution of the bits in  $B$ conditioned on $D_p$. If
the bit arriving at time $\tau$ is part of a segment of $rep_{n,k}$ copies of the same random
bit, then each sequence in $\Delta_B$ starts with the suffix of this segment starting at time $\tau$. The sequence then continues as in $\Delta$ with  alternation between a single $Ber(1/2)$ bit, and  $rep_{n,k}$ copies of another $Ber(1/2)$ bit.
Since $rep_{n,k} = rep_{Big_{n,k},k-1}= rep_{\abs{B},k-1}$ we have that 
 $\Delta_B\in \Delta_{\abs{B},k-1}$.  
 
  Note that when $M'_{B}$ runs on the input distribution $\Delta_B$, the inputs that it feeds to $M'$ during batch $B$ has the same distribution as in $\Delta$ conditioned on the prefix $D_p$. 
  By our assumption towards contradiction,  all sum approximates of  $M'$ during $B$ have error $O\left(rep_{n,k}\right)$ 
  with probability $> 1-P_{n,k-1}$ over $D$ and $S$ conditioned on $D_p$ and $S_p$, respectively, so $M'_B$ has error $O\left(rep_{n,k}\right)$ over all the prefixes with probability 
 $> 1-P_{n,k-1}$, over the randomness of $\Delta_B$ and its shuffle randomness. This contradicts the induction hypothesis for $n\gets \abs{B}$ and $k\gets k-1$, which claims that specifically for $M'_B$ and $\Delta_B\in \Delta_{\abs{B},k-1}$, we get error $\Omega(\abs{B}^{\frac{1}{2k-1}}\cdot c_\eps^{\frac{2k-2}{2k-1}})=\Omega(n^{\frac{2k-1}{(2k+1)(2k-1)}}\cdot c_\eps^{\frac{2}{(2k+1)(2k-1)}}\cdot c_\eps^{\frac{2k-2}{2k-1}})=\Omega(rep_{n,k})$ at some time with probability $\geq P_{\abs{B},k-1}\geq P_{n,k-1}$.
Hence, our assumption was false, and $M'$ has error $\Omega\left(rep_{n,k}\right)$ at some time $t$ with probability $\geq P_{n,k-1}$ over $D$ and $S$.


    
    

\medskip
To conclude the proof we denote by $ERROR_M$ and $ERROR_{M'}$ the worst error of $M$ and $M'$ respectively over all prefixes, and aggregate the error bound for the different sets by the law of total probability, to get that
\begin{align*}
    \Pr_{D,S}[ERROR_{M}= \Tilde{\Omega}\left(rep_{n,k}\right)]&=
    \Pr_{D,S}[ERROR_{M'}= \Tilde{\Omega}\left(rep_{n,k}\right)]\\
    &=\Pr_{D,S}[ERROR_{M'}= \Tilde{\Omega}\left(rep_{n,k}\right) \mid T'\in T'_{small,n,k}\cap \overline{E'}]\cdot \Pr_{D,S}[T'\in T'_{small,n,k}\cap \overline{E'}]\\
    &+\sum_{t'\in T'_{small,n,k}\cap E'}\Pr_{D,S}[ERROR_{M'}= \Tilde{\Omega}\left(rep_{n,k}\right) \mid T' =t']\cdot \Pr_{D,S}[T' =t']\\
    &+\sum_{p}\Pr_{D,S}[ERROR_{M'}= \Tilde{\Omega}\left(rep_{n,k}\right) \mid T' \in T_p]\cdot \Pr_{D,S}[T' \in T_p]\\
    &\geq 0+\beta_0/2\cdot \Pr_{D,S}[T' \in T'_{small,n,k}\cap E']+P_{n,k-1}\cdot \Pr_{D,S}[T' \in \cup_{p}T_p]\\
    &\geq P_{n,k-1}\cdot \Pr_{D,S}[T' \in (T'_{small,n,k}\cap E') \cup (\cup_{p}T_p)]\\
    &= P_{n,k-1}\cdot \Pr_{D,S}[T' \notin T'_{small,n,k}\cap \overline{E'}]\\
    &\geq P_{n,k-1}\cdot \left(1-\frac{2n\delta}{\min(\eps,1)}\right) = \beta_0/2\cdot \left(1-\frac{2n\delta (k-1)}{\min(\eps,1)}\right)\cdot \left(1-\frac{2n\delta}{\min(\eps,1)}\right)\\
    &> \beta_0/2\cdot \left(1-\frac{2n\delta k}{\min(\eps,1)}\right) = P_{n,k},
\end{align*}
where the first step follows by the definition of $M'$, the second step follows by the law of total probability, the third step follows by Lemma~\ref{lem:NoConditionDprime} and by the 
analysis of transcripts of
type 2  (resulting with the constant $\beta_0/2$) above, the fourth step follows since $P_{n',k'}\leq \beta_0/2$ for any $n'$, $k'$, and the sixth step follows by the analysis of set 1, and the seventh step follows by the definition of $P_{n,k-1}$.
\end{proof}

We now provide the missing lemma used in the proof of  Theorem~\ref{thm:lb}, which shows using the notation of the proof of Theorem~\ref{thm:lb}, that for each transcript $t'\in T'_{small,n,k}\cap E'$, conditioned on $T'=t'$, $M'$ has additive error $\alpha = \Omega(rep_{n,k})$.
\begin{lemma}\label{lem:NoConditionDprime}
    Let $\eps$ and $\delta$ such that $\delta<\frac{\min(\eps,1)}{40n}$ and $nc_{\eps}^2/Big_{n,k}>200^2$, and an $(\eps,\delta)$-CSDP $n$-user $k$-shuffler algorithm $M$ for \psco{} with adaptive mechanisms, and let $M'$ and $E'$ as defined in the proof of Theorem~\ref{thm:lb}. Let  $\Delta\in \Delta_{n,k}$ be a distribution over user binary values. Then, \textbf{conditioned on a fixed transcript $T'=t'=(t,\overline{Bits_{n,k}}) \in T'_{small,n,k}\cap E'$}, $M'$ has additive error $\alpha = \Omega(rep_{n,k})$ at some time with  probability $\beta_0 / 2$, over the randomness of $\Delta$ and $M'$.
\end{lemma}
\begin{proof}[Proof of Lemma~\ref{lem:NoConditionDprime}]

        We claim that:
        \begin{enumerate}
            \item The conditional random variables $\{X_i\mid T'=t'\}_{i=1}^{m}$ are independent.
            (Recall that  $Bits_{n,k}:=\{X_i\}_{i=1}^{m}$ are the non-repeated $Ber(1/2)$ bits of $D\sim \Delta\in \Delta_{n,k}$ which are evenly spaced at distance $Big_{n,k}$ of one another. )
            \item Each bit $X_i$ satisfies $\Pr_{D,S}[X_i=1 \mid T'=t']\in \left[\frac{1}{e^{2\eps}+1}, \frac{e^{2\eps}}{e^{2\eps}+1}\right]$.
        \end{enumerate}
        
        We prove Item (a) through an application of  Lemma~\ref{lem:independence}. Recall that a random $D\sim \Delta$ is composed of a constant (non-random)  prefix of bits, then alternating single $Ber(1/2)$ bits and a sequence of equal $Ber(1/2)$ bits. 
        
        The communicating parties can be divided to those which are associated with input bit subsets or a specific batch of users (defined by the subset of users, the mechanism and the shuffler they used). For input bits, we associate a communicating party with each (a) single $Ber(1/2)$ bit in the input, (b) sequence of of copies of a $Ber(1/2)$ bit, and (c) a constant bit in the prefix of at most $rep_{n,k}$ bits. For batches of users, we associate a communicating party of type (d) to each batch of users which contain no bit among $Bits_{n,k}$.\footnote{The outputs of the shuffle mechanisms over batches containing some bit of $Bits_{n,k}$ are associated with this single $Bits_{n,k}$ input bit from above. Note that we can make this association since the batches are of size at most $Big_{n,k}$ since $t'\in T'_{small,n,k}$, so the batch cannot contain two or more consecutive $X_j$ and $X_{j+1}$'s.}

        The communication protocol is divided to $2n$ rounds of communication, two rounds per each time $1,\ldots,n$. 
        Throughout the times $t=1,\ldots,n$, the first communication rounds are used to communicate bits of the input. Each party which is associated with a set of input bits communicates them in the first round of a single step, as follows. A party associated with a single bit (types (a),(c)), communicates this bit at the first round of time $t$ of the bit, and a party associated with a sequence of copies of a bit (type (b)), communicates the copied bit value at the first round of time $t$ which contains the last copied bit of the sequence.
        Throughout the times $t=1,\ldots,n$, the second communication rounds are used to communicate shuffler outputs. Specifically, in each time $t$ which contains a bit of $Bits_{n,k}$, the party associated with the bit $t$ (this is a subset of parties of type (a) as above) communicates the output of all (at most $k$) shufflers that closed in time $t$ in the second communication round of $t$. Also, for each batch of type (d), its associated party communicates the output of the shuffler run on the batch at the second round of the shuffler's close time $t$.
        
        We show that this is a valid communication protocol, that is, each party communicates a message as a function of their internal randomness, their private input and the history. For parties which communicate in the first rounds, the communicated message is simply their bit which is their private input. For the second rounds of communication, each participating party associated with a $Bits_{n,k}$ bit, e.g., $X_i$, communicates the output of a shuffle mechanism, which is a function of the shuffle input $X_i$ (attributed to the input of party $i$), the rest of the bits of users in the shuffle batch (which were communicated previously or in the first round of the current time), the executed shuffle mechanism (which depends on the information of previous rounds of communication) and the shuffle randomness (attributed to the internal randomness of party $i$). A similar argument is given for parties of type (d) which are not associated with any $Bits_{n,k}$ bit. 
        
        We observe that the bits $\{X_i\}_i$ were initially independent by the definition of $\Delta$, and that the total communication of the parties precisely produces the transcript $t'$, and therefore Item (a) follows by Lemma~\ref{lem:independence}.
        
        To prove item (b), for every $i=1,\ldots,m$ we observe that
        \begin{align*}
            \frac{\Pr_{D,S}[X_i=1 \mid T'=t']}{\Pr_{D,S}[X_i=0 \mid T'=t']}&= \frac{\Pr_{D,S}[T'=t'\mid X_i=1]\cdot \Pr_{D,S}[X_i=1 \mid ]/\Pr_{D,S}[T'=t' \mid ]}{\Pr_{D,S}[ T'=t'\mid X_i=0]\cdot \Pr_{D,S}[X_i=0 \mid ]/\Pr_{D,S}[T'=t'\mid ]} \\
            &= \frac{\Pr_{D,S}[T'=t'\mid X_i=1]\cdot 1/2}{\Pr_{D,S}[ T'=t'\mid X_i=0]\cdot 1/2} = \frac{\Pr_{D,S}[T'=t'\mid X_i=1]}{\Pr_{D,S}[ T'=t'\mid X_i=0]}\\
            &\in [e^{-2\eps}, e^{2\eps}]
            \label{eq:p}
            ,
        \end{align*}
        where the first step follows by Bayes’ Rule, the second step follows since for $j\in \{0,1\}$, $\Pr_{D}[X_i=j]=0.5$ and since $X_i$ is independent of $S$, 
        and the last step follows by since $t'\in E'$ and by the definition of $E'$. This concludes the proof of item (b).
        
        Combining Claims 1 and 2 above, we have that $\{X_i\mid T'=t'\}_{i=1}^{m}$ are independent, and each bit $X_i$ satisfies $\Pr_{D,S}[X_i=1 \mid T'=t']\in \left[\frac{1}{e^{2\eps}+1}, \frac{e^{2\eps}}{e^{2\eps}+1}\right]$.
        Thus, $\sum_{i=1}^{m} X_i$ is the sum of $m$ independent Bernoulli random variables with bias in $\left[\frac{1}{e^{2\eps}+1}, \frac{e^{2\eps}}{e^{2\eps}+1} \right]$ (so the variance is at least $ c_{\eps}^2$) and therefore $\sum_{i=1}^{m} X_i$ has variance at least $mc_{\eps}^2>200^2$, where the last step follows by the theorem's assumptions on $n$ and $\eps$.
        
        
        We apply Lemma~\ref{lem:matouvsek} to the random variable $\left(\sum_{i=1}^{m}X_i\mid T'=t'\right)$, with the constant $\beta=\beta_0=\min(\frac{a}{4},1/4)\in (2^{-bmc_\eps^2},a)$ where $a$ is a constant guaranteed from Lemma~\ref{lem:matouvsek}, and the interval $I(t')=M'(t')-sum(\overline{Bits_{n,k}})\pm c/2\cdot \sqrt{mc_\eps^2\cdot \log (\frac{a}{\beta_0})}$, where $M'(t')$ is the (deterministic) last sum estimate that $M'$ returns for the transcript $t'$ (which is precisely the last estimate that $M$ returns for $t$), and $sum(\overline{Bits_{n,k}})$ is the sum of bits of $D$ which are not of type $Bits_{n,k}$.\footnote{The reason we use this interval is that the true sum is $trueSum = sum(\overline{Bits_{n,k}})+\sum_{i=1}^{m}X_i$, and to show that with high probability we mis-approximate it with error $\geq \phi$ with $M'(e)$, it suffices to show that $M'(e)-sum(\overline{Bits_{n,k}})-\sum_{i=1}^{m}X_i$ mis-approximates 0 with error $\geq \phi$.}
        We conclude that with constant probability $\beta_0$ over $S$ and $D$, $M'$ incurs error at least $\Tilde{\Omega}(\sqrt{m}\cdot c_\eps)=\Tilde{\Omega}(\sqrt{n/Big_{n,k}}\cdot c_\eps) = \Tilde{\Omega}(rep_{n,k})$, by the definition of $Big_{n,k}$ and $rep_{n,k}$. 
\end{proof}

\section{Concurrent Shuffle LinUCB}\label{sec:concShufflerLinUCBAppendix}
In this section, we present with detail the LinUCB adaptation for the Concurrent Shuffle model with $k$ shufflers. Algorithm~\ref{alg:concurrentshufflerlinucb} below simulates LinUCB with errors, as in the setting of \citet{shariff2018differentially}, where the estimates of the cumulative matrix and vector sums are given by a CSDP algorithm for \psco{} guaranteed from Corollary~\ref{thm:treealgorithmSummationKShufflers}.
We make a standard assumption of boundedness, easily satisfied by normalization as follows.
\begin{assumption}[normalization]\label{assumption:norm}
By normalization, we assume that the rewards are bounded for all $t$, i.e., $y_t\in [0,1]$. Also, $\norm{\theta^*}_2\leq 1$ and $\sup_{c,a}\norm{\phi(c,a)}_2\leq 1$.
\end{assumption}

\begin{algorithm}[H]
\SetAlgoLined
\textbf{Parameters:} horizon $n$, number of shufflers $k$, privacy parameters $\eps$ and $\delta$, regularization $\lambda >0$, confidence radii $\{\beta_i\}_{i=1}^{n}$, feature map $\phi: C \times \chi \to \reals^d$.\;

\textbf{Initialize:} Batch statistics $V_0=\lambda I_d$, $u_0=0$, parameter estimate $\hat{\theta_0}=0$, and a Concurrent Shuffle DP (CSDP) bounded vector and matrix summation algorithm $A^{sum}_{vec}$ guaranteed by Corollary~\ref{thm:treealgorithmSummationKShufflers} for $n$ users, $k$ shufflers, and privacy parameters $\eps$ and $\delta$.\;

\footnotesize{\textcolor{gray}{We wlog assume that $A^{sum}_{vec}$ handles both vectors and matrix sums, since they use the same tree structure (same users in each mechanism by the definition of the tree-based structure in Corollary~\ref{thm:treealgorithmSummationKShufflers}), and the two shuffle mechanisms can be composed using simple composition.}}

For each time $t=1,\ldots,n$:
\begin{enumerate}
    \item The server activates the new shuffle mechanisms for the $t$'th time, according to $A^{sum}_{vec}$.\;
    \item The server communicates $\hat{\theta}_{m-1}$ and $V^{-1}_{m-1}$ to the $t$'th user.\;
    \item The $t$'th user chooses the action $a_t\in \argmax_{a \in \chi} \ip{\phi(c_t,a)}{\hat{\theta}_{m-1}}+\beta_{t-1}\cdot\norm{\phi(c_t,a)}_{V^{-1}_{m-1}}$.\;
    \item The $t$'th user observes the reward $y_t$.\;
    \item For each shuffle mechanism which requires $t$'s participation for vectors and matrices sums, it participates using its private vector value $\phi(c_t,a_t)y_t$ and matrix value $\phi(c_t,a_t)\phi(c_t,a_t)^T$ respectively.\;
    \item The data of each shuffle mechanism that has just completed is shuffled and sent to the server.\;
    \item The server analyzes the data to obtain a cumulative sum vector estimate $u_t$ and a cumulative matrix estimate $V'_t$, and computes the parameter estimates $V_t = V'_t+\lambda I_d$ and $\hat{\theta}_t=V^{-1}_t u_t$.\;
\end{enumerate}
 \caption{Concurrent Shuffle Private LinUCB}
 \label{alg:concurrentshufflerlinucb}
\end{algorithm}

To prove the regret of Algorithm~\ref{alg:concurrentshufflerlinucb}, we refer to the framework of \citet{shariff2018differentially}, who gave a bound on the total regret of an algorithm based on the magnitude of the errors of the estimates $u_t$ and $V^{-1}_t$  throughout the algorithm.
Specifically, they used the following notations, where for each time $t$, they denote by $H_t:=V_t-\sum_{i=1}^{t}\phi(c_t,a_t)\phi(c_t,a_t)^T$ the noise injected into the Gram-matrix, and denote by $h_t:=u_t-\sum_{i=1}^{t}\phi(c_t,a_t)y_t$ the noise injected into the feature-reward vector.

\citet{shariff2018differentially} showed the following.
\begin{theorem}[\citet{shariff2018differentially}]\label{ss18}
If we have that
\begin{enumerate}
    \item \textbf{(Regularity)} For any $\alpha\in (0,1]$, $H_m$ is PSD, and there are $\rho_{\max}\leq \rho_{\min}>0$ and $\gamma$ depending on $\alpha$, such that with probability at least $1-\alpha$, for all $t=1,\ldots,n$,
$$\norm{H_t}\leq \rho_{\max},~~\norm{H^{-1}_t}\leq 1/\rho_{\min},~~\norm{h_t}_{H^{-1}_t}\leq \gamma.$$
    \item \textbf{(Boundedness)} $\abs{\ip{\theta^*}{x}}\leq B$, $\norm{\theta^*}\leq S$, $\norm{x}\leq L$ and $y_t=\ip{\theta^*}{x_t}+\eta_t$ where $\eta_t$ is sub-gaussian with variance $\sigma^2$.
\end{enumerate}
Then with probability at least $1-\alpha$, the regret satisfies:
$$Reg(n)\leq B\sqrt{8n}\cdot \left(\sigma \left(2\log (2/\alpha)+d\log \left(\frac{\rho_{\max}}{\rho_{\min}}+\frac{nL^2}{d\rho_{\min}}\right)\right)+(S\sqrt{\rho_{\max}}+\gamma)\sqrt{d\log\left(1+\frac{nL^2}{d\rho_{\min}}\right)}\right),$$
using $\beta_t:= \sigma\sqrt{2\log (2/\alpha)+d\log \left(\frac{\rho_{\max}}{\rho_{\min}}+\frac{tL^2}{d\rho_{\min}}\right)}+S\sqrt{\rho_{\max}}+\gamma$.
\end{theorem}

We are now prepared to prove the main result of this section.

\begin{theorem}[Restatement of Theorem~\ref{thm:contextualLinearKShuffles}]
    Fix a horizon $n\in \naturals$, number of shufflers $k$ and privacy budgets $\eps<1$ and $\delta < 1/2$.  
    Then the algorithm as described above instantiated with the tree-Based Concurrent Shuffle algorithm (guaranteed by Corollary~\ref{thm:treealgorithmSummationKShufflers} for summing vectors with bounded $l_2$ norms) with internal parameter $\beta_t:= \sigma\sqrt{2\log (2/\alpha)+d\log \left(3+\frac{2t}{d\lambda}\right)}+\sqrt{3\lambda/2}+\sqrt{\lambda/2}$, where $\alpha=\frac{1}{n}\in (0,1]$, $\lambda = 2\sigma'\cdot d$ and $\sigma'=\Tilde{O}\left(\frac{k^{3/2}}{\eps}\cdot n^{\frac{1}{2k+1}}\right)$ \jay{Should I use the accurate $\sigma'$ from the proof of Corollary~\ref{thm:treealgorithmSummationKShufflers}, or is this enough?} is $(\eps,\delta)$-CSDP and has regret $\tilde{O}\left(\frac{k^{3/4}n^{\frac{k+1}{2k+1}}}{\sqrt{\eps}}\cdot (\sigma +d)\right)$, where $\sigma^2$ is the bound on the sub-gaussian variance of each $\eta_t$ when computing $y_t=\ip{\theta^*}{x_t}+\eta_t$. 
    
\end{theorem}
\begin{proof}
Our proof is similar to the proof of Lemma A.4 of \citet{chowdhury2022shuffle}. However, we use different constants and internal variable values, so include the proof below.

We receive the bound on the regret by applying Theorem~\ref{ss18}. To do so, we find the variables which ensure regularity and boundedness. 

For regularity, note that for all $t=1,\ldots,n$, $h_t$ is a random vector whose entries are independent, and $H_t$ is a random symmetric matrix whose entries on and above the diagonal are independent. 
We apply Corollary~\ref{thm:treealgorithmSummationKShufflers} to get that every entry of every matrix and vector at each time $t$ has error at most $\sigma'$ with probability $\geq \frac{1}{n^2\cdot (d^2+d)}$. 
By a union bound over $n$ and the entries of the matrix and vectors, and a simple bound on vector and matrix $l_2$ norms, we get that with probability at least $1-\alpha/2$, for all $t=1,\ldots,n$ simultaneously, $\norm{h_t},\norm{H_t}\leq \sigma'd=\lambda/2$. Hence, we have for all $t=1,\ldots,n$ , it holds that $\norm{H_t}=\norm{H_t+\lambda I_d}\leq 3\lambda/2$, i.e., $\rho_{\max}=3\lambda/2$, and $\rho_{\min}=\lambda/2$. Finally, to determine $\gamma$, we note that $\norm{h_t}_{H^{-1}_t}\leq \frac{1}{\sqrt{\rho_{\min}}}\cdot \norm{h_t}\leq \sqrt{\sigma'\cdot d}=\sqrt{\lambda/2}:=\gamma$. 

For boundedness, by the normalization Assumption~\ref{assumption:norm} we get boundedness for $B=S=L=1$, and we use the same the definition of $\sigma$ as from Section~\ref{sec:contextualLinearBandits}.

We now apply Theorem~\ref{ss18} with the regularity and boundedness results above for the $\alpha$ and observe that the theorem's statement contains the correct matching sequence $\beta_t$. We get that with probability at least $1-\alpha$, the regret satisfies:
\begin{align*}
    Reg(n)&\leq \sqrt{8n}\cdot \left(\sigma \left(2\log (2n)+d\log \left(3+\frac{2n}{d\lambda}\right)\right)+(\sqrt{3\lambda/2}+\gamma)\sqrt{d\log\left(1+\frac{2n}{d\lambda}\right)}\right)\\
&=\tilde{O}\left(\sqrt{n}\cdot \left((\sigma+d)\log(n)+\sqrt{\sigma'}\cdot\sqrt{d}\cdot\sqrt{d\log(n)}\right)\right)\\
&=\tilde{O}\left(\sqrt{n}\cdot \left(\sigma\log(n)+d\sqrt{\sigma'}\log(n)\right)\right)
=\tilde{O}\left(\sqrt{n}\cdot \left(\sigma\log(n)+d\sqrt{\frac{k^{3/2}}{\eps} \cdot n^{\frac{1}{2k+1}}}\log(n)\right)\right)\\
&=\tilde{O}\left(\frac{k^{3/4}n^{\frac{k+1}{2k+1}}}{\sqrt{\eps}}\cdot (\sigma +d)\log (n)\right)=\tilde{O}\left(\frac{k^{3/4}n^{\frac{k+1}{2k+1}}}{\sqrt{\eps}}\cdot (\sigma +d)\right),
\end{align*}
where the fourth step follows by the definition of $\sigma'$ from the theorem's statement.

Since the worst possible regret is $n$, splitting to the event above with probability at least $1-\alpha=1-\frac{1}{n}$ or its complement for $\alpha=\frac{1}{n}$, the law of total expectation, the regret is at most $$Reg(n)\leq \frac{1}{n}\cdot n +\left(1-\frac{1}{n}\right)\cdot \tilde{O}\left(\frac{k^{3/4}n^{\frac{k+1}{2k+1}}}{\sqrt{\eps}}\cdot (\sigma +d)\right)=\tilde{O}\left(\frac{k^{3/4}n^{\frac{k+1}{2k+1}}}{\sqrt{\eps}}\cdot (\sigma +d)\right).$$
\end{proof}

\begin{proof}[Proof of Corollary~\ref{cor:contextualLinearLogShuffles}]
    The proof follows by a simple application of Theorem~\ref{thm:contextualLinearKShuffles} for $k=\log n$.
\end{proof}
\jay{
\section{TODO}
\begin{enumerate}
    \item Check "TODO(jay) - I replaced it" in the sty file
    \item Use shuffler/shuffle/mechanism/batch/protocol consistently!
    \begin{itemize}
        \item A \textbf{shuffler} is the entity which contains messages and shuffles.
        \item A \textbf{shuffle mechanism} is the algorithm which runs on the shuffler.
        \item For each shuffler $S$, we divide the users differently into \textbf{batches}, and run a different shuffle mechanism on each batch using $S$.
        \item How call a specific run of the mechanism on a specific set of users and on a specific shuffler? A run of the mechanism over a batch.
    \end{itemize}
\end{enumerate}
}

\end{document}